\documentclass{article}

\usepackage{microtype}
\usepackage{graphicx}
\usepackage{subfigure}
\usepackage{bbm}
\usepackage{booktabs} 
\usepackage{enumitem}
\setlist[itemize]{itemsep=0mm}
\setlist[enumerate]{itemsep=0mm}
\usepackage{xcolor,colortbl}
\usepackage{xspace}

\usepackage{hyperref}

\usepackage[accepted]{icml2024}

\usepackage{amsmath}
\usepackage{amssymb}
\usepackage{mathtools}
\usepackage{amsthm}

\usepackage[capitalize,noabbrev]{cleveref}

\theoremstyle{plain}
\newtheorem{theorem}{Theorem}[section]

\newtheorem{lemma}[theorem]{Lemma}
\newtheorem{corollary}[theorem]{Corollary}
\theoremstyle{definition}
\newtheorem{definition}[theorem]{Definition}

\theoremstyle{remark}

\makeatletter
\newcommand{\newreptheorem}[2]{\newtheorem*{rep@#1}{\rep@title}\newenvironment{rep#1}[1]{\def\rep@title{#2 \ref*{##1}}\begin{rep@#1}}{\end{rep@#1}}}
\makeatother

\definecolor{lucky}{RGB}{120, 130, 150}

\definecolor{newgreen}{rgb}{0.0, 0.5, 0.0}
\definecolor{newred}{rgb}{0.81,0.1,0.26}
\definecolor{babyblue}{rgb}{0.54, 0.81, 0.94}
\definecolor{negcol}{rgb}{0.96, 0.76, 0.76}
\definecolor{poscol}{rgb}{0.63, 0.79, 0.95}
\definecolor{sunwoogreen}{rgb}{0.66, 0.89, 0.63}
\definecolor{sunwoored}{rgb}{1.0, 1.0, 0.0}
\definecolor{sunwoogreentwo}{rgb}{0.53, 0.81, 0.98}
\definecolor{dodgerblue}{rgb}{0.0, 0.75, 1.0}
\definecolor{orange}{rgb}{1.0, 0.8, 0.0}
\definecolor{crimson}{rgb}{0.86, 0.08, 0.24}
\definecolor{limegreen}{rgb}{0.2, 0.8, 0.2}
\definecolor{backredcolor}{rgb}{0.91, 0.45, 0.32}
\definecolor{secgreen}{rgb}{0.3, 1, 0.0}

\newcommand{\best}{\cellcolor{dodgerblue}}  
\newcommand{\secb}{\cellcolor{secgreen}}  
\newcommand{\thib}{\cellcolor{sunwoored}}  

\newcommand{\method}{M2M-GNN\xspace }

\usepackage[textsize=tiny]{todonotes}

\usepackage[font=small,skip=2pt]{caption}

\icmltitlerunning{Multiset-to-Multiset Message Passing}

\begin{document}

\twocolumn[
\icmltitle{Sign is Not a Remedy: Multiset-to-Multiset Message Passing for Learning on Heterophilic Graphs}



\icmlsetsymbol{equal}{*}

\begin{icmlauthorlist}
\icmlauthor{Langzhang Liang}{hit}
\icmlauthor{Sunwoo Kim}{kaist}
\icmlauthor{Kijung Shin}{kaist}
\icmlauthor{Zenglin Xu}{fdu,sais,hit,pcl}
\icmlauthor{Shirui Pan}{griffith}
\icmlauthor{Yuan Qi}{fdu,sais}
\end{icmlauthorlist}

\icmlaffiliation{hit}{Harbin Institute of Technology, Shenzhen, Shenzhen, China}
\icmlaffiliation{kaist}{Korea Advanced Institute of Science and Technology, Seoul, South Korea}
\icmlaffiliation{pcl}{Pengcheng Laboratory, Shenzhen, China}
\icmlaffiliation{griffith}{Griffith University, Gold Coast, Australia}
\icmlaffiliation{fdu}{Fudan University, Shanghai, China}
\icmlaffiliation{sais}{Shanghai Academy of Artificial Intelligence for Science, Shanghai, China}

\icmlcorrespondingauthor{Langzhang Liang}{lazylzliang@gmail.com}
\icmlcorrespondingauthor{Zenglin Xu}{zenglin@gmail.com}

\icmlkeywords{Machine Learning, ICML}

\vskip 0.3in
]



\printAffiliationsAndNotice{}  

\begin{abstract}

\urlstyle{same}

Graph Neural Networks (GNNs) have gained significant attention as a powerful modeling and inference method, especially for homophilic graph-structured data.
To empower GNNs in heterophilic graphs, where adjacent nodes exhibit dissimilar labels or features, Signed Message Passing (SMP) has been widely adopted.
However, there is a lack of theoretical and empirical analysis regarding the limitations of SMP. 
In this work, we unveil some potential pitfalls of SMP and their remedies.
We first identify two limitations of SMP: undesirable representation update for multi-hop neighbors and vulnerability against oversmoothing issues.
To overcome these challenges, we propose a novel message-passing function called Multiset to Multiset GNN (M2M-GNN).
Our theoretical analyses and extensive experiments demonstrate that M2M-GNN effectively alleviates the aforementioned limitations of SMP, yielding superior performance in comparison. 

\end{abstract}

\section{Introduction}
\label{sec:intro}
Graph Neural Networks (GNNs)~\citep{defferrard2016convolutional,kipf2016semi,hamilton2017inductive, velivckovic2018graph} have emerged as a prominent approach for machine learning on graph-structured data, such as social networks~\citep{DBLP:conf/ijcai/TangHGL13} and biomedical networks~\citep{lin2020kgnn}.
In essence, GNNs employ message passing to iteratively aggregate information from neighboring nodes, ultimately yielding node embeddings for downstream tasks.

Despite their success, the effectiveness of GNNs often diminishes on heterophilic graphs, where neighboring nodes typically have dissimilar labels or features.
Surprisingly, GNNs can even be outperformed by graph-agnostic models like MLPs~\citep{zhu2020beyond, DBLP:journals/corr/abs-2202-07082}. 
This issue arises from the inherent smoothing effect of message passing, which tends to make adjacent nodes have similar embeddings.

To address this limitation, Signed Message Passing (SMP) has been widely adopted~\citep{bo2021beyond,chien2020adaptive, yan2022two,DBLP:conf/icml/LeeBYS23}. 
SMP incorporates negative weights in message passing to push neighboring nodes apart in the embedding space. Consequently, SMP is also recognized as an effective solution for alleviating a key drawback of GNNs known as oversmoothing~\citep{li2018deeper}. It refers to the tendency of node embeddings to become indistinguishable as multiple GNN layers are stacked.


However, in this work, we uncover two significant limitations of SMP.
First, we prove that, even when the message-passing weights are ``desirable'' for direct neighbors, the cumulative weights for multi-hop neighbors may not necessarily be desirable. Second, we reveal that, even when message-passing weights are desirable, SMP can still be vulnerable to oversmoothing.
More importantly, we point out that these limitations of SMP arise from its simple aggregator, which linearly combines the embeddings of all neighbors of each node into a single message vector.



Motivated by our findings, we propose a novel message passing scheme for GNNs, which we refer to as \textbf{\method} (\textbf{M}ultiset 
\textbf{to} \textbf{M}ultiset GNN), to address the limitations of SMP. In essence, \method aggregates neighborhood information by mapping the embeddings of neighbors, which can be considered a multiset of vectors, to another multiset.
This differs from existing message passing schemes, where a multiset of vectors is mapped to a single message vector. We prove \method's theoretical properties related to alleviating the limitations of SMP together with its empirical effectiveness. Our code is available at \url{https://github.com/Jinx-byebye/m2mgnn}

Our contributions are summarized as follows:
\begin{itemize}
    \item \textbf{Analysis:} Our theoretical analyses reveal two fundamental limitations of SMP.
    \item \textbf{Method:}  We introduce \method, a novel multiset-to-multiset message passing scheme, and illustrate how it mitigates the aforementioned limitations.   
    \item \textbf{Experiments:} We conduct comprehensive experiments that demonstrate the superior performance of \method.
\end{itemize}



\section{Limitations of Signed Message Passing}{
    \label{sec:limitation}
    In this section, we introduce two limitations of signed message passing (SMP): (1) {undesirable} representation update and (2) {vulnerability} to over-smoothing.
To this end, we begin by introducing relevant notations and preliminaries.
We provide all proofs in Appendix~\ref{appendix:fullproof}.

\subsection{Preliminaries}

\textbf{Basic terminologies.}
For any positive integer $N\in \mathbb{N_+}$, we denote the set $\{1, 2, \ldots, N\}$ as $[N]$.
For any matrix $\mathbf{M}$, we use $\mathbf{M}_{i} = \mathbf{m}_{i}$ to denote its $i$-th row, and we use $\mathbf{M}_{ij}$ to denote the element in the $i$-th row and $j$-th column of $\mathbf{M}$.
Let $\mathcal{G=(V, E, \mathbf{X})}$ {denote} an undirected graph, where $\mathcal{V}=\{v_1, v_2,\cdots,v_N\}$ represents the node set and $\mathcal{E}$ represents the edge set. 
The nodes are characterized by a node feature matrix $\mathbf{X} \in \mathbb{R}^{N\times d}$, where $d$ denotes the number of features per node. 
Let $\mathbf{y} \in [C]^{N}$ be the node label matrix where $y_{i} \in [C]$ denotes the label of node $v_{i}$, and $C$ denotes the number of classes.
A node pair with the same label is termed \textit{homophilic nodes}, while a node pair with distinct labels is called \textit{heterophilic nodes}.
Here, the edge set $\mathcal{E}$ can be represented as an adjacency matrix $\mathbf{A} \in \left\{0,1\right\}^{N\times N}$ where $\mathbf{A}_{ij} = 1$ if $\{v_{i},v_{j}\} \in \mathcal{E}$ and $\mathbf{A}_{ij} = 0$ otherwise.

\textbf{Message passing.}
Given a graph $\mathcal{G}$ with its adjacency matrix $\mathbf{A}$, \textit{Message Passing} (\texttt{MP}) involves iteratively updating node embeddings by aggregating neighbors' information. The $d'$-dimensional node embeddings at each $k$-th layer, denoted as $\mathbf{H}^{(k)} \in \mathbb{R}^{N \times d'}$, can be expressed as follows:
\begin{equation}\label{eq:1}
    \mathbf{H}^{(k)} = \texttt{MP}(\mathbf{A}, \mathbf{H}^{(k-1)}, \mathbf{W}) = \sigma\left(\mathbf{A}\mathbf{H}^{(k-1)}\mathbf{W}\right),
\end{equation}
where $\sigma(\cdot)$ denotes a (non-linear) activation function, and $\mathbf{W}$ denotes a learnable weight matrix.
Instead of directly employing $\mathbf{A}$, many architectures utilize its variant denoted as $\mathcal{A}$, which we call a \textbf{\textit{propagation matrix}}. Examples of $\mathcal{A}$ include normalized $\mathbf{A}$ in~\citet{kipf2016semi} ($\mathcal{A}=\mathbf{D}^{-0.5}\mathbf{A}\mathbf{D}^{-0.5}$ with $\mathbf{D}$ being the diagonal degree matrix) and learnable ones in Graph Attention Networks~\citep{velivckovic2018graph} (where non-zero elements are computed using adjacent nodes' representations). We call the element $\mathcal{A}_{ij}$ the \textbf{\textit{propagation coefficient}} for edge $\{v_i,v_j\}$. To clarify, Eq. \eqref{eq:1} represents the message passing frameworks utilizing the mean aggregator as its aggregation function. There exist other design choices. We use this definition as Eq. \eqref{eq:1} is sufficiently general to analyze SMP models.

\textbf{Signed message passing.}
\textit{Signed Message Passing} (SMP) enables $\mathcal{A}$ to incorporate both positive and negative values~\citep{chien2020adaptive,bo2021beyond,luan2022revisiting,yan2022two}. By introducing more flexibility in $\mathcal{A}$, SMP facilitates adaptive updates of node embeddings in heterophilic graphs, as elaborated in Section~\ref{subsec:lim1-undesirable}. 

\subsection{Limitation 1: Undesirable Embedding Update}\label{subsec:lim1-undesirable}

We first outline the mechanism of SMP and its particular benefits in heterophilic graphs.
Subsequently, we show that, contrary to our expectation, this mechanism does not consistently yield beneficial updates of node embeddings.

\textbf{Desirable embedding update.}
For accurate node classification, how should node embeddings be updated?
In general, nodes with similar embeddings are more likely to be classified into the same class. 
Thus, intuitively, in an {ideal} scenario, embeddings of homophilic nodes would be closely located, while embeddings of heterophilic nodes would be relatively distant from each other.
``Desirable'' SMP (i.e., desirable $\mathcal{A}$) facilitates this circumstance by assigning positive and negative propagation coefficients to homophilic and heterophilic node pairs, respectively. This assignment is based on the understanding that positive coefficients tend to increase similarity between the embedding of pairs, while negative coefficients have the opposite effect. This phenomenon has been observed in numerous prior studies~\citep{bo2021beyond, yan2022two, choi2023signed}.


Based on the suggested {ideal} scenario, a desirable propagation matrix is defined as follows:
\begin{definition}[{Desirable} matrix]\label{def:good}
    A matrix $\mathbf{M} \in \mathbb{R}^{N \times N}$ is \textbf{\textit{desirable}} if and only if
    the following holds: for any $v_i,v_j \in \mathcal{V}$, $\mathbf{M}_{ij}\ge 0$, if $y_i=y_j$; $\mathbf{M}_{ij}\le 0$, if $y_i\neq y_j$. 
\end{definition}

\textbf{Message passing scheme.}
In practice, multi-layer GNNs are commonly employed to expand the receptive field of graph convolution.
Following prior analyses~\citep{chien2020adaptive,bo2021beyond} we consider linear SMP with multiple layers, formalized as:
\begin{align}\label{eq:final_embedding}
    \mathbf{H}^{(K)} & = \mathcal{A}^{(K)}\cdots\mathcal{A}^{(1)}\mathbf{H}^{(0)}  \nonumber 
    \\ & = \prod\nolimits_{k=0}^{K-1}\mathcal{A}^{(K-k)}\mathbf{H}^{(0)} \coloneqq \mathcal{T}\mathbf{H}^{(0)},
\end{align}
where $\mathcal{A}^{(k)}$ is a (learnable) propagation matrix for the $k$-th layer, $\mathcal{T}:=\prod_{k=0}^{K-1}\mathcal{A}^{(K-k)}$ is the \textbf{\textit{cumulative propagation matrix}}, and $\mathbf{H}^{(0)} = f_\theta(\mathbf X)$ with $f_\theta(\cdot)$ being a neural network parameterized by $\theta$. 


Typically, matrices $\mathcal{A}^{(k)}$'s have the same sparsity pattern of the adjacency matrix $\mathbf{A}$. Note that the sparsity pattern identifies the specific coordinates within the matrix that contain non-zero elements. Thus, their product, $\mathcal{T}$, has the same sparsity pattern of $\mathbf{A}^K$, i.e., the $K$-th power of the adjacency matrix, where $\mathbf{A}_{ij}^K \neq 0$ only if $v_i$ and $v_j$ have a distance at most $K$. 
That is, $\mathcal{A}^{(k)}$ and $\mathcal{T}$ can have non-zero propagation coefficients for direct neighbors and $K$-hop reachable neighbors, respectively. 

\textbf{Theoretical result.}
Under the aforementioned design choice of SMP, one would expect that with desirable $\mathcal{A}^{(k)},\forall k \in [K]$, the resulting $\mathcal{T}$ should also be desirable.
Surprisingly, our analysis reveals the existence of counterexamples that are prevalent in the real worlds, specifically graphs with more than two distinct node classes.
\begin{theorem}[Undesirability of SMP in multi-class cases]\label{thm:smp_in_multi_class}
    There exists a graph $\mathcal G$ with more than two distinct node classes (i.e., $C > 2$) where all propagation matrix $\mathcal{A}^{(1)},\mathcal{A}^{(2)},\cdots , \mathcal{A}^{(K)}$ are \textbf{desirable} (Def.~\ref{def:good}), but the cumulative propagation matrix $\mathcal{T}$ is \textbf{not desirable}.
\end{theorem}
The presence of undesirable $\mathcal{T}$ implies that, in the embedding space, nodes may end up being close to their heterophilic multi-hop neighbors, resulting in similar embeddings and, consequently, classification into the same class.




\subsection{Limitation 2: Vulnerability to Oversmoothing}
Next, we introduce another major limitation of SMP, i.e., its susceptibility to oversmoothing. Previous works have primarily focused on the potential advantages of SMP in mitigating oversmoothing, particularly its capacity to utilize negative propagation coefficients to separate nodes within the embedding space~\citep{chien2020adaptive,bo2021beyond,yan2022two}.
However, through our theoretical and empirical investigation, we verify that even in a highly ideal setting, SMP is still susceptible to oversmoothing.


Below, we first introduce the analysis setup and then present our theoretical findings.
Lastly, we review simulations designed to validate our theoretical findings.

\begin{figure*}[t]
    \vspace{-4mm}
    \subfigure[binary case]{\includegraphics[width=0.67\columnwidth]{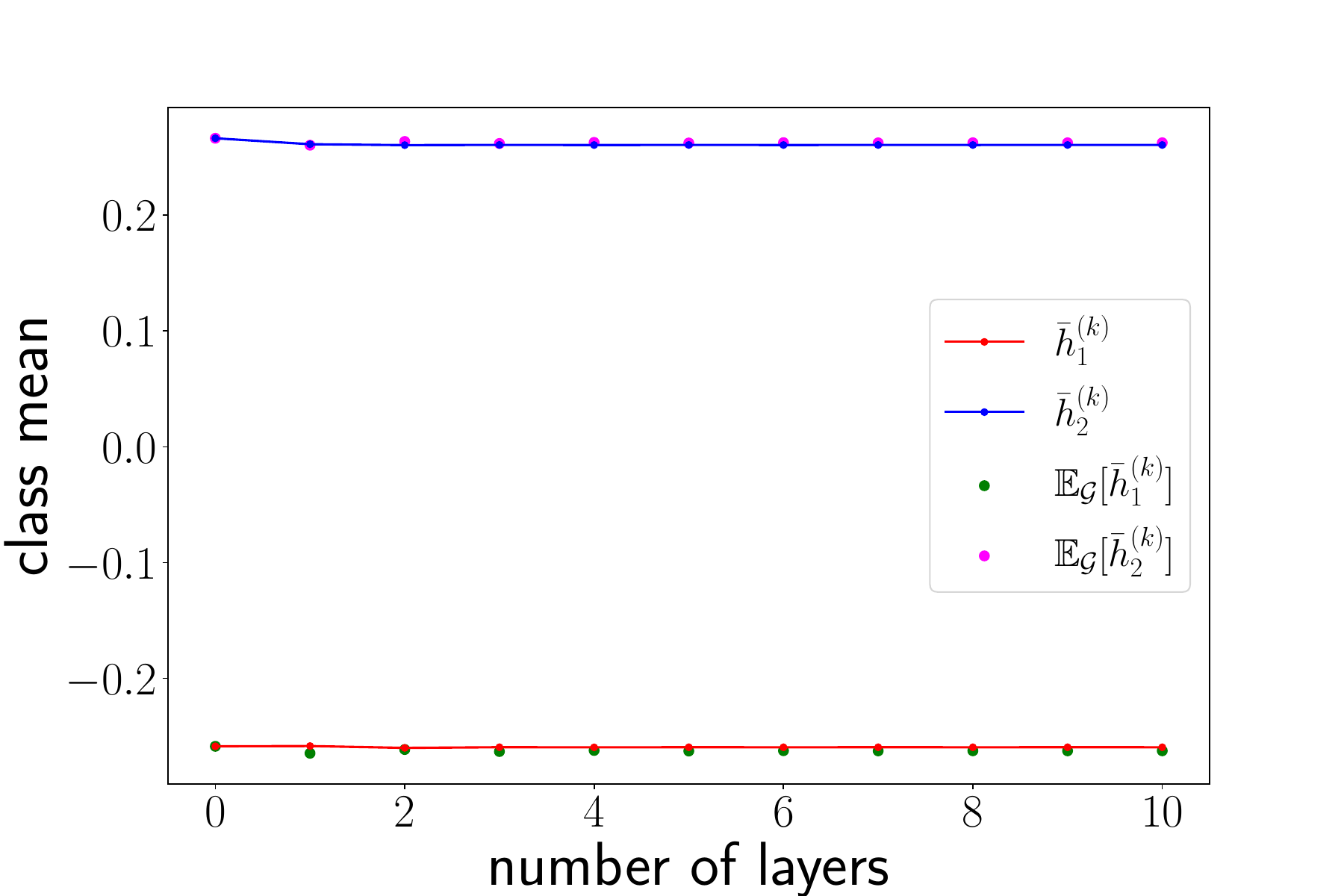}}    
    \subfigure[multi-class ($C=3$) case]{\includegraphics[width=0.67\columnwidth]{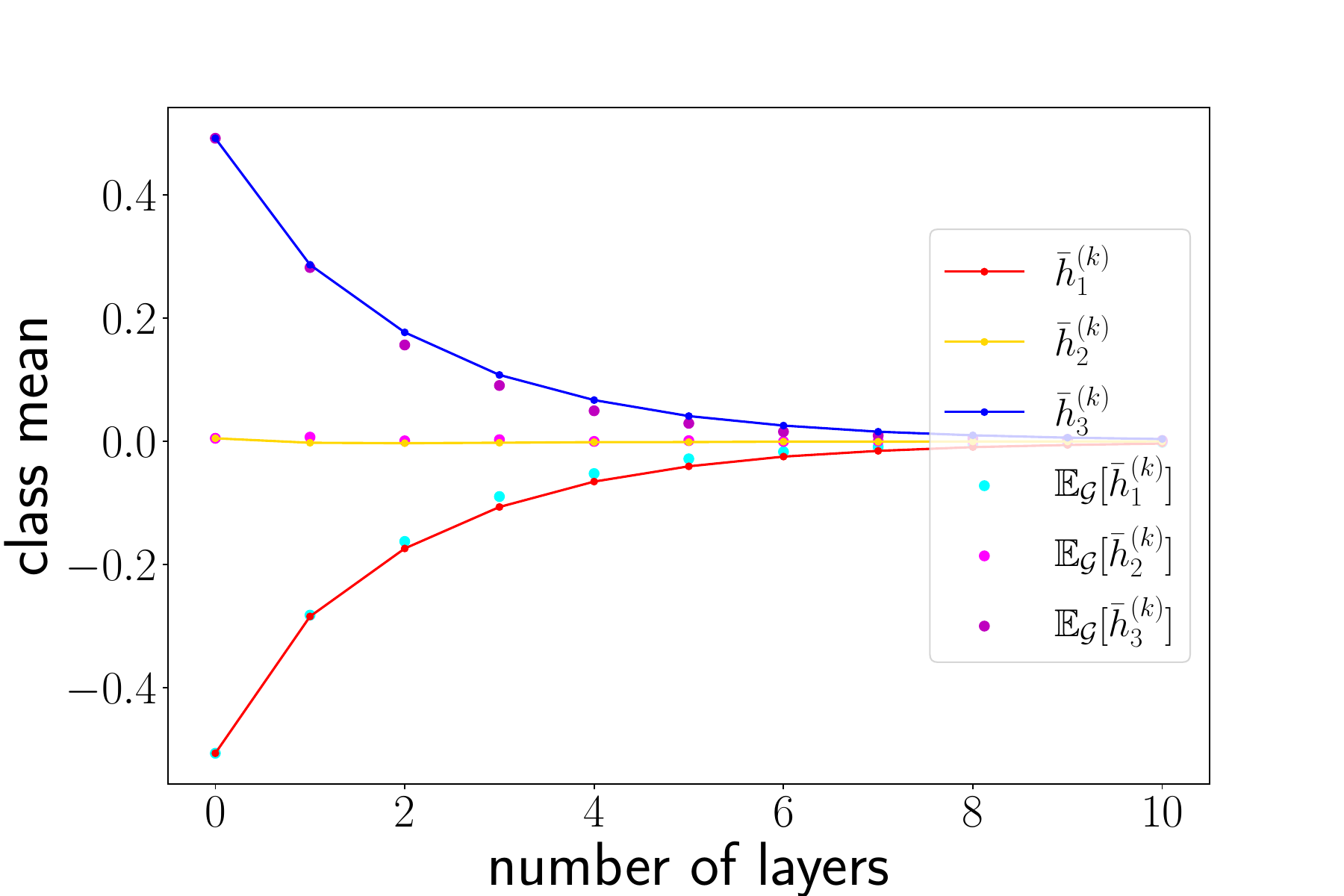}}
    \subfigure[score]{\includegraphics[width=0.67\columnwidth]{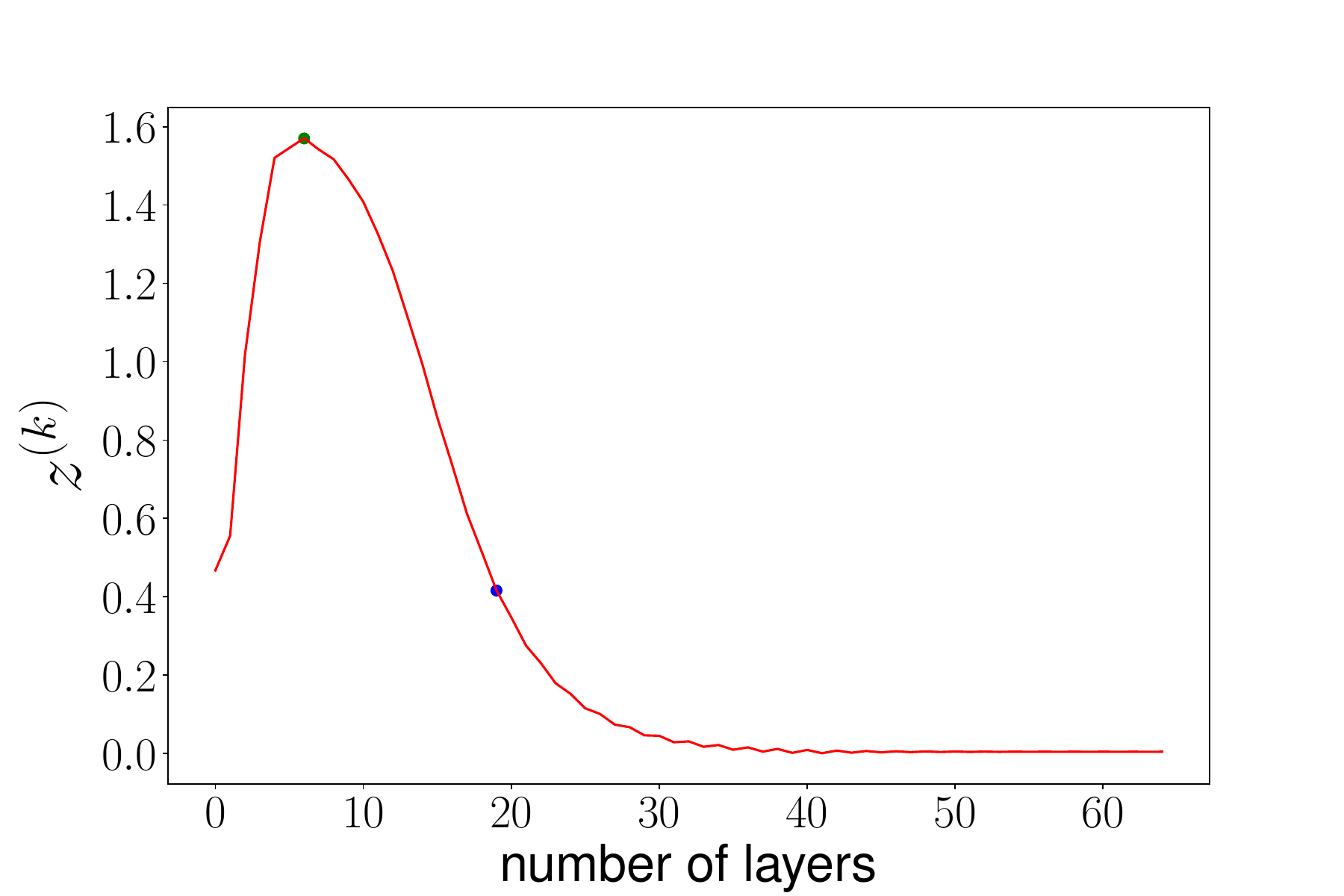}} \\
    \vspace{-2mm}
    \caption{Simulation results for (a) binary and (b) multi-class classification reveal that, as the number of layers increases, the mean embeddings of different classes converge in the multi-class case but not in the binary case.
    That is, SMP suffers from oversmoothing in the multi-class case, as also supported by (c) the drop of $z^{(k)}$, a score positively associated with the accuracy of the Bayes optimal classifier.
    }
    \label{fig: simulation}
\end{figure*}

\textbf{Setup: (1) Propagation matrix and features.} 
In this analysis, we assume a desirable propagation matrix $\mathcal{A}$ and node features $\mathbf{X}$ generated by a variant of the Contextual Stochastic Block Model (CSBM)~\citep{deshpande2018contextual}, where each node feature is generated from a particular Gaussian distribution corresponding to its class, and edges are sampled independently from Bernoulli distributions. CSBMs have been widely employed for theoretical analyses of graphs~\citep{,wu2022non,DBLP:journals/corr/abs-2402-04621}

\begin{definition}[CSBM for $\mathcal{A}$ and $\mathbf{X}$]\label{def:csbm}
    Assume $N$ nodes and $C$ classes, each having an equal number of nodes.
    For each node $v_{i} \in \mathcal{V}$, its feature vector $\mathbf{x}_{i}$ is sampled from $\mathcal{N}(\mathbf{u}_{y_{i}}, \Sigma) \in \mathbb{R}^{f}$.
    For each node pair $\{v_{i},v_{j}\} \in \binom{\mathcal{V}}{2}$, (1) $\mathcal{A}_{ij}=1$, if $(y_{i} = y_{j}) \vee (\mathcal{U} < p)$ holds, and (2) $\mathcal{A}_{ij}=-1$, if $(y_{i}\neq y_{j}) \vee (\mathcal{U} < q)$ holds, where $\mathcal{U} \sim \text{uniform(0,1)}$ is sampled independently for each node pair. 
    As a result, we have a desirable propagation matrix $\mathcal{A} \in \{-1,0,1\}^{N \times N}$ and a node feature matrix $\mathbf{X} \in \mathbb{R}^{N \times f}$. 
\end{definition}


\textbf{Setup: (2) Message passing.}
Following prior oversmoothing analyses~\citep{zhou2021dirichlet,keriven2022not,wu2022non,DBLP:conf/icml/LeeBYS23}, we employ a simple message passing scheme, spec., the product of weighted (for SMP, signed) propagation matrices without non-linear activation.
Formally, given a signed propagation matrix $\mathcal{A} \in \{-1,0,1\}^{N \times N}$, we normalize it as $\mathcal{P}\coloneqq {D}^{-\frac{1}{2}}\mathcal{A}D^{-\frac{1}{2}}$, where $D$ is a diagonal matrix with $D_{ii} = \sum_{j=1}^{N}\vert \mathcal{A}_{ij}\vert, \forall i\in [N]$. 
Then, the node embedding matrix at the $K$-th layer is:
\begin{equation}\label{eq:theoryembedding}
    \mathbf{H}^{(K)} = \mathcal{P}\mathbf{H}^{(K-1)} = \mathcal{P}^{K}\mathbf{H}^{(0)} \coloneqq  \mathcal{P}^{K}\mathbf{X}.
\end{equation}
As SMP in our theoretical analysis, we use Eq.~\eqref{eq:theoryembedding}, which is very similar to FAGCN~\citep{bo2021beyond}.
It has a desirable property: $\mathcal{L} = \mathcal{I}-\mathcal{P}$ is a well-defined Laplacian of signed graphs~\citep{atay2014spectrum}, which ensures that the spectral norm of $\mathcal{P}^{K}$ is bounded, even when infinite layers are stacked.
In Sec.~\ref{sec:experiment}, we present relevant experimental results using popular SMP-based methods and real-world graphs.

\textbf{Theoretical result.}
With CSBM (Def.~\ref{def:csbm}) and SMP (Eq.~\eqref{eq:theoryembedding}), we present our theoretical finding: SMP is vulnerable to the over-smoothing problem.
\begin{theorem}[Oversmoothing problem of SMP]\label{thm:oversmoothing}
    Consider random variables $\mathcal{A}$ and $\mathbf{X}$ from CSBM (Def.~\ref{def:csbm}), and recall that $\mathcal{A}$ is always desirable.
    Let $\mathcal{V}_{c} \coloneqq \{v_{j} \in \mathcal{V} : y_{j} = c\}$ be the set of nodes belonging to each class $c\in[C]$, and 
    let $\bar{\mathbf{h}}^{(K)}_{c} \coloneqq \frac{C}{N}\sum_{v_{i} \in \mathcal{V}_{c}}\mathbf{h}^{(K)}_{i}$ be the mean of the $K$-th layer embeddings of $\mathcal{V}_{c}$. 
    Then, for any two classes $a, b\in [C]$, 
    \begin{equation}\label{eq:dist_form}
        \mathbb{E}_{\mathcal{G}} [\lVert \bar{\mathbf{h}}^{(K)}_{a} - \bar{\mathbf{h}}^{(K)}_{b}\rVert] = \left(\frac{p + q}{p + (C-1)q}\right)^{K}\lVert \mathbf{u}_{a} - \mathbf{u}_{b} \rVert.
    \end{equation}
\end{theorem}
Refer to Appendix \ref{app:concentration} for a concentration bound of $\Vert(\bar{\mathbf{h}}^{(K)}_{a} - \bar{\mathbf{h}}^{(K)}_{b})-\mathbb{E}_{\mathcal{G}} [ \bar{\mathbf{h}}^{(K)}_{a} - \bar{\mathbf{h}}^{(K)}_{b}]\Vert$.
Theorem~\ref{thm:oversmoothing} states that, even when $\mathcal{A}$ is desirable, if $C>2$ (i.e., multi-class), the expected distance between the average embeddings of any pair of classes can decrease exponentially as the number of layers $K$ increases.
That is, the separability of different class embeddings significantly decreases as more layers are stacked, revealing SMP's susceptibility to oversmoothing.

\textbf{Simulation result.} 
Fig.~\ref{fig: simulation} shows the simulation results for binary-class and multi-class ($C=3$) cases with 20 pairs of $\mathcal{A}$ and $\mathbf{X}$ from CSBM (Def.~\ref{def:csbm}).
For the binary case, we set $\mathbf u_1=-0.25$, $\mathbf u_2=0.25$, and for the multi-class case, we set $\mathbf u_1=-0.5$, $\mathbf u_2=0$, and $\mathbf u_3=0.5$. 
The other parameters are $N=3000$, $p=0.003$, and $q=0.01$.
One-dimensional features are used for illustrative purposes. We have several observations.
(1) The expected and actual class means are close, validating that Theorem \ref{thm:oversmoothing} is meaningful.
(2) SMP is robust to oversmoothing in the binary-class case, i.e., the difference between two class means does not decrease.
(3) SMP suffers from oversmoothing in the multi-class case, with the differences between any two class means decreasing as the number of layers increases.
(4) We further use a numerical score $z^{(k)}$ measuring the distinguishability between different classes, i.e., class 1 
and class 2.\footnote{The score is defined as $z^{(k)} =(\bar{\mathbf{h}}_{2}^{(k)}- \bar{\mathbf{h}}_{1}^{(k)})/ \sigma^{(k)}, \sigma^{(k)} = (\sigma_1^{(k)}+\sigma_2^{(k)})/2$, where $\sigma_1^{(k)}$ and $\sigma_2^{(k)}$ denote the variances of features from class 1 and class 2, respectively.}
The score is positively correlated with the accuracy of the optimal Bayes classifier~\citep{wu2022non}, where a higher value is desired, implying that the two classes are easier to distinguish.
As the depth increases, we observe an initial rise in the score, followed by rapid decline and ultimate convergence to zero, illustrating the typical phenomenon of oversmoothing.

While our analysis of SMP's limitations is founded on linear models, these findings can be seamlessly applied to non-linear SMP models. See Appendix~\ref{app:non-linearity} for a discussion.

}

\section{Multiset to Multiset Message Passing}{
    \label{sec:method}
In this section, we propose \textbf{\method} (\textbf{M}ultiset \textbf{to} \textbf{M}ultiset GNN),
our novel message-passing paradigm that provably mitigates the aforementioned limitations of SMP.

\subsection{Motivation of \method}\label{subsec:m2mgnn_motivation}
Let us first recall the limitations we observed on SMP.
First, a discrepancy arises between one-hop and multi-hop aggregation: even if we have a desirable one-hop propagation matrix $\mathcal{A}^{(k)}$s, 
the multi-hop propagation matrix $\mathcal{T}$ can still be undesirable (Theorem~\ref{thm:smp_in_multi_class}).
Second, even with negative propagation coefficients between them, the embeddings of different classes can be mixed up, thereby leading to oversmoothing (Theorem~\ref{thm:oversmoothing}).
All proofs are in Appendix~\ref{appendix:fullproof}.

\textbf{Limitations of simple pooling approach.}
We point out that, even with desirable weights (or signs), the simple weighted summation scheme acts as the information bottleneck of SMP.
Specifically, the weighted summation has a functional form of \textit{multiset-to-element} (\textit{m-2-e}).
In the aggregation step, multiple elements (the embeddings of neighbors) are reduced (i.e., aggregated) into a single element (a single message vector), and the embeddings of heterophilic nodes are thus mixed up.
Formally, an \textit{m-2-e} message passing function can be represented as $\mathbf{m}_i = \phi(\mathcal{S}_{i})$,
where $\mathcal{S}_{i} \coloneqq \left\{\mathbf{h}_j: v_j \in \mathcal{N}(v_i) \right\}$,\footnote{Precisely, $\mathcal{S}^{k}_{i}$ consists of $k-1$ layer outputs of $v_{i}$ neighbors. For simplicity, superscripts denoting layers are omitted.} 
and $\phi(\cdot)$ is a mapping function defined as follows.
\begin{definition}\label{def:mapping}
    A mapping function $\phi(\cdot)$ takes a set of vectors as input and produces a single element as output, by applying a learnable weight matrix to each element of the set, followed by element-wise pooling (sum, mean, and max) to aggregate all elements. 
\end{definition}
For instance, using sum pooling and learnable weight matrix $\mathbf{W}$ yields a mapping $\phi(\mathcal{S}_{i})=\sum_{\mathbf h_j \in \mathcal{S}_{i}} (\mathbf h_j \mathbf{W})$.




\textbf{Multiset to multiset: Partition, mapping, and concatenation.}
Then, what if we aggregate embeddings into multiple elements instead?
In other words, we can take a functional form of \textit{multiset-to-multiset} (\textit{m-2-m}).
Specifically, given a multiset of node embeddings,\footnote{It is a multiset because different nodes may have identical embeddings, and each of them should be kept.} an \textit{m-2-m} scheme
(1) partitions the multiset into multiple subsets,
(2) maps the features in each subset into an element (e.g., a vector), and 
(3) combines the elements into a multiset.

\begin{definition}[\textit{m-2-m} schemes]\label{def:m2m}
    Given a multiset $\mathcal{X}$ where each $x_{i} \in \mathcal{X}$ is a vector $x_i \in \mathbb{R}^{d}$,
    an \textit{m-2-m} scheme $f$ does the following:
    (1) it partitions $\mathcal{X} = \bigcup_{t = 1}^{\omega} \mathcal{X}_t$
    into $\omega$ disjoint subsets for some $\omega \in \mathbb{N}$,
    (2) it applies a mapping 
    $\phi \colon 2^{\mathcal{X}} \mapsto \mathbb{R}^{d'}$ on each subset $\mathcal{X}_t$ to obtain an element (a vector) $\mathbf{z}_t = \phi(\mathcal{X}_t)$ for each subset, and
    (3) the obtained elements essentially form a multiset, and it concatenates the obtained elements into
    $f(\mathcal{X}) = \mathbf{z} = \mathbin\Vert_{t = 1}^{\omega} \mathbf{z}_t$, where
    $\mathbin\Vert$ denotes the vector concatenation operator.
\end{definition}

\textbf{Desirable \textit{m-2-m} schemes.}
There can be different \textit{m-2-m} schemes, e.g., we have diverse ways of partitioning a given multiset $\mathcal{X}$.
What kinds of \textit{m-2-m} schemes are \textit{desirable}?
Recall that the limitations of SMP primarily arise from the intermixing of various classes.
Hence, our high-level idea is to ensure that heterophilic node representations are not intertwined, i.e., when we partition $\mathcal{X}$, each $\mathcal{X}_k$ only consists of the embeddings of nodes from the same class.
Specifically, we define \textit{desirable m-2-m schemes} as follows.

\begin{definition}[Desirable \textit{m-2-m} schemes]\label{def:desirable_m2m}
    Assume for any multiset $\mathcal{X}$, each vector $x_i \in \mathcal{X}$ associated with a class $y_i$, we say
    an \textit{m-2-m} scheme is \textit{desirable}, iff,
    for any given $\cal{X}$, it always partitions $\mathcal{X}$ according to the class of $x_i$'s, i.e.,
    it partitions $\mathcal{X} = \bigcup_{t = 1}^{\omega} \mathcal{X}_t$ such that,
    if two vectors $x_i, x_j \in \mathcal{X}$ are in the same $\mathcal{X}_t$, then we must have $y_i = y_j$.
    Note that we do \textit{not} necessarily have $x_i, x_j$  in the same $\mathcal{X}_t$ if $y_i = y_j$.
\end{definition}
Based on the concept of desirable \textit{m-2-m} schemes, we further define \textit{desirable \textit{m-2-m} message passing}.

\begin{definition}[Desirable \textit{m-2-m} message passing]\label{def:desirable_m2m_mespas}
    An \textit{\textit{m-2-m} message passing} operation $f_{mp}$ is \textit{desirable}, 
    if given any node embedding matrix $\mathbf{H}$ and any node $v_i$, there exists a mapping $\phi(\cdot)$ such that the message vector of $v_i$ can be represented as
    \begin{equation}\label{eq:m2m_message_passing_overall}
     f_{mp}(\mathbf{H}, \phi(\cdot); i) = \mathbf{m}_{i} = 
     \mathbin\Vert_{t=1}^{\omega} \phi(\mathcal{S}_{i,t}),  
    \end{equation}
    where 
    (1) the union of $\mathcal{S}_{i,t}$'s is a subset of $\mathbf{H}$, i.e., $\bigcup_{t=1}^{\omega} \mathcal{S}_{i,t} \subseteq \{\mathbf{h}_j \colon j \in [N]\}$, and
    (2) each $\mathcal{S}_{i,t}$ contains embeddings of nodes from the same class,\footnote{Or $\mathcal{S}_{i,t} = \emptyset$, and we let $\phi(\mathcal{S}_{i,t}) = \mathbf{0}$, so that all the $\mathbf{m}_{i}$'s have the same dimension.} i.e., if $\mathbf{h}_j, \mathbf{h}_{j'} \in \mathcal{S}_{i,t}$ then $y_j = y_{j'}$.
    

    Moreover, we say $f_{mp}$ is a \textit{\textbf{one-hop} desirable} \textit{m-2-m} message passing operation
    if there exists a mapping function $\phi(\cdot)$ satisfying
    \begin{equation}\label{eq:m2m_message_passing}
     f_{mp}(\mathbf{H} , \phi(\cdot); i) = \mathbf{m}_{i} = \mathbin\Vert_{t=1}^{C} \phi(\mathcal{S}_{i,t}),
    \end{equation}
    where $\mathcal{S}_{i,t} \coloneqq \{\mathbf{h}_j:y_{j} = t,v_j \in \mathcal{N}(v_{i})\}$.
\end{definition}

\textbf{Desirable \textit{m-2-m} message passing alleviates the limitations.}
Below, we shall analyze several theoretical properties of \textit{m-2-m} message passing and how it mitigates the limitations of SMP (i.e., \textit{m-2-e}) discussed in Section~\ref{sec:limitation}.

\textbf{\textit{First}}, regarding desirable properties, we claim that stacking one-hop \textit{desirable} \textit{m-2-m} message passing operations always gives us a desirable (multi-hop) message passing.
Recall that this does not hold for SMP (Theorem~\ref{thm:smp_in_multi_class}).

{

\begin{lemma}[Maintenance of desirable property]\label{lem:desirablem2m}
    Assuming a \textbf{one-hop} desirable \textit{m-2-m} message passing $f_{mp}^{(k)}$ is applied to each node at each layer $k$.
    Specifically, for each layer $k$,    
    the embedding of each node $v_{i}$ at the $k$-th layer is $\mathbf{h}^{(k)}_{i} = f_{mp}^{(k)}(\mathbf{H}^{(k-1)}, \phi^{(k)}(\cdot); i) =
    \mathbin\Vert_{t=1}^{C} \phi^{(k)}(\mathcal{S}_{i,t}^{(k)})$,
    where $\mathcal{S}_{i,t}^{(k)} = \{\mathbf{h}_j^{(k-1)}:v_j \in \mathcal{N}(v_i), y_j = t\}$.
    Then, for any $d, k \in \mathbb{N}$,  
    the $d$-hop message passing operation at the $k$-th layer
    $\rho^{(d, k)}$ defined by
    $\rho^{(d, k)}(\mathbf{H}^{(k)}, \Phi^{(k+d)}(\cdot); i) = \mathbf{h}_i^{(k + d)}$, 
    is desirable.
    Specifically, $\rho^{(d, k)}$ stacks $d$ one-hop message passing operations from $f_{mp}^{(k + 1)}$ to $f_{mp}^{(k + d)}$, and it partitions the layer-$k$ embeddings of the $d$-hop neighbors of each node, where each group contains the embeddings of $d$-hop neighbors in the same class.
\end{lemma}

}
Lemma~\ref{lem:desirablem2m} demonstrates that unlike SMP (\textit{m-2-e}), the proposed \textit{m-2-m} message passing can generalize the desirable property from local to global.
This is because, following our high-level idea, at each layer, each group consists of embeddings from the same class, and thus heterophilic node representations are not intertwined even when we stack multiple message passing operations.



\textbf{\textit{Second}}, regarding oversmoothing, we show that the concatenation utilized in desirable \textit{m-2-m} message passing can enhance the discriminability among message vectors. 
We present informal statements that describe the robustness of \textit{m-2-m} to oversmoothing phenomena, with the formal statements and proofs detailed in Appendix~\ref{sec:proof_for_oversmoothing}.

\begin{lemma}[\textit{m-2-m} is always no worse than \textit{m-2-e} (informal)]\label{theorem:discri_pow}
Given any two embedding multisets, the distance between the two resulting message vectors of \textit{m-2-m} is greater than or equal to the distance between those of \textit{m-2-e}.
\end{lemma}

\begin{theorem}[\textit{m-2-m} never gets trapped by oversmoothing (informal)]\label{thm:robust-to-oversmoothing}
    Under the same CSBM settings as in Sec.~\ref{sec:limitation}, a one-hop \textit{m-2-m} message passing (Eq. ~\eqref{eq:m2m_message_passing} in Def.~\ref{def:desirable_m2m_mespas})
    can escape from oversmoothing even when the expected class means converge to the same point, i.e., even if $\mathbb E_{\mathcal{G}} [\bar{\mathbf{h}}_a^{(k-1)}] = \mathbb E_{\mathcal{G}} [\bar{\mathbf{h}}_b^{(k-1)}]$ we have $\mathbb E_{\mathcal{G}} [\bar{\mathbf{h}}_a^{(k)}] \neq \mathbb E_{\mathcal{G}} [\bar{\mathbf{h}}_b^{(k)}]$ as long as $p \neq q$,
    while SMP fails to achieve this.
\end{theorem}
We refer the reader to Appendix~\ref{app:discriminative_power} for a discussion on the discriminative power of \textit{m-2-m}.




\begin{figure*}[htp]
    \centering
    \includegraphics[width=0.8\textwidth]{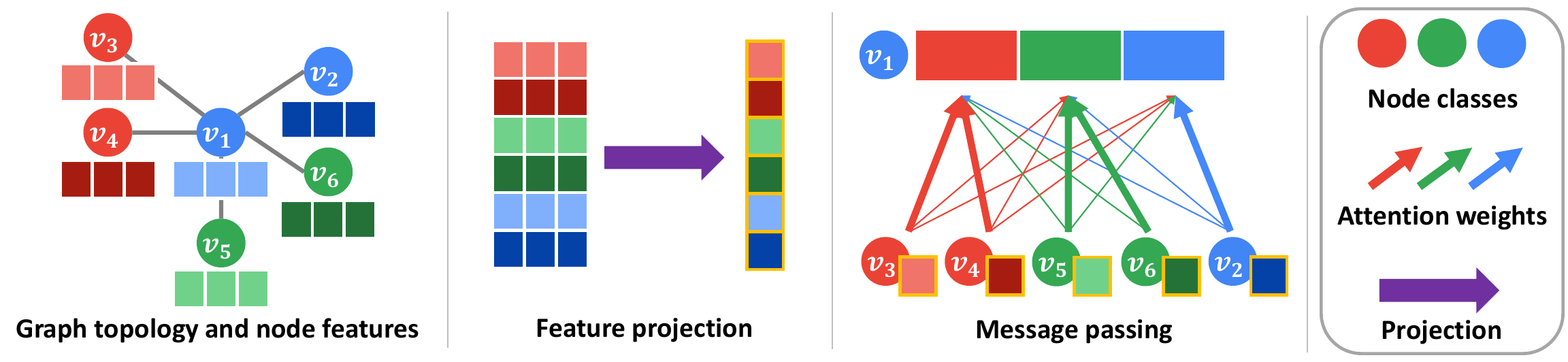}
    \caption{An illustration of the convolution layer of \method, where the ego (central) node is $v_{1}$.
    For a given graph topology and node features (left), we first project node features from $3$-dimension to $1$-dimension (middle).
    Then, we construct a message vector for node $v_{1}$ by concatenating $3$ chunks, each of which consists of a subset of neighbors' projected features using soft labels}.
    \label{fig:method}
\end{figure*}

\subsection{Proposed Method: \method}\label{subsec:proposed_method}
Motivated by the good properties of \textit{m-2-m} message passing, we present a specific instance it, termed \textbf{\method}. 
An illustration of \method is presented in Figure~\ref{fig:method}.

\textbf{Core idea: Chunks.}
Each partition is represented as a \textit{chunk}.
Ideally, a model should assign a label to each chunk while ensuring a one-to-one correspondence between chunk labels and node classes, i.e.,
each chunk should only contain nodes belonging to the class associated with that chunk.


\textbf{Method details.}
We elaborate on how \method updates node embeddings using $\mathcal{C}$ chunks. Note that $\mathcal{C}$ may not necessarily be the same as $C$, which represents the number of classes.
We focus on the embedding update at the $k$-th layer.
For a given $(k-1)$-th layer output node embeddings $\mathbf{H}^{(k-1)} \in \mathbb{R}^{N \times d}$, \method first projects the node embeddings using a learnable weight matrix $\mathbf{W}^{(k)} \in \mathbb{R}^{d \times (d/\mathcal{C})}$.
Formally, $\hat{\mathbf{H}}^{(k)}=\mathbf{H}^{(k-1)}\mathbf{W}^{(k)}$.


Next, we describe the message passing scheme of \method. 
We define $v_{i}$ as an ego (central) node and $\mathcal N(v_{i}) \coloneqq N_{i}$ as the set of neighbors of $v_{i}$. 
Since in practice, we typically do not have the ground-truth class information of nodes, we propose to utilize soft labels that are obtained via an attention function.
We use the soft labels to assign each node (and its embedding) to chunks.
In a nutshell, every node embedding of $N_{i}$ is assigned to every chunk, but with different scores (i.e., weights).
To learn the assigned scores, we employ the attention function inspired by~\citet{brody2021attentive}.
Specifically, for $v_{j} \in \mathcal N_{i}$, its scores towards $v_{i}$ are defined as follows:
\begin{equation}\label{eq:attention}
    \mathbf{s}^{(k)}({i},{j}) = \texttt{Softmax}(\texttt{ReLU}(\alpha\hat{\mathbf{h}}^{(k)}_{i} + \hat{\mathbf{h}}^{(k)}_{j})\mathbf{W}^{(k)}_{att}/\tau),
\end{equation}
where $\tau$ is a temperature hyperparameter, $0<\alpha<1$ is a hyperparameter that weights the importance of the ego node $v_i$'s representation, and $\mathbf{W}^{(k)}_{att} \in \mathbb{R}^{(d/\mathcal{C}) \times \mathcal{C}}$ is the learnable attention weight.
Note that $\mathbf{s}^{(k)}({i},{j}) \in \mathbb{R}^{\mathcal{C}}$ indicates scores of $v_{j}$ during message passing towards $v_{i}$, where $\mathbf{s}_{t}^{(k)}({i},{j})$ is the score assigned to the $t$-th chunk.
Moreover, these scores can be interpreted as a soft label of $v_{j}$, and in the ideal case, \method assigns a high score to the chunk that corresponds to the class of $v_{j}$, and nearly zero values to other chunks.
We examine this property in Section~\ref{subsec:exp_visual}.

\begin{equation}
    \mathbf{C}_{it}^{(k)} = \sum_{v_{j} \in N_{i}}\mathbf{s}_{t}^{(k)}(i,j)\hat{\mathbf{h}}^{(k)}_{j} \in \mathbb{R}^{d/\mathcal{C}}, \forall t \in [\mathcal{C}].
\end{equation}
Then, the message vector of $v_{i}$, denoted by $\mathbf{m}^{(k)}_{i}$, is defined as the concatenation of $\mathbf{C}^{(k)}_{i1},\cdots ,\mathbf{C}^{(k)}_{i\mathcal{C}}$, 
i.e., $\mathbf{m}^{(k)}_{i} = \lVert_{t=1}^{\mathcal{C}}\mathbf{C}_{it}^{(k)}$.
The dimension of $\mathbf{m}^{(k)}_{i}$ is equal to that of $\mathbf{h}^{(k-1)}_{i}$, which is $\mathbb{R}^{d}$.


Finally, the node representation of $v_{i}$ is updated using the following weighted average:
\begin{equation}
    \mathbf{h}^{(k)}_{i} = \texttt{ReLU}((1 - \beta)\mathbf{h}^{(0)}_{i} + \beta \mathbf{m}^{(k)}_{i}),
\end{equation}
where $\beta \in (0,1)$ is a hyperparameter, and $\mathbf{h}^{(0)}_{k}$ is the transformed original feature of $v_{i}$ (i.e., $\mathbf{h}^{(0)}_{i} = \texttt{MLP}(\mathbf{x}_{i})$).

\textbf{Regularizing soft labels.}
In our preliminary study, we observed that every attention score is often collapsed into one particular chunk (i.e., there exists a $1\le t \le \mathcal{C}$ such that $s_t^{(k)}(i,j)\approx 1$ for all edges $\{v_i,v_j\}$).
This deviates from a desirable scenario, where attention scores of nodes are diversified to different chunks according to their respective classes.
Motivated by ~\citet{tsitsulin2023graph}, we leverage a regularization term to help \method diversify the attention scores. Unlike ~\citet{tsitsulin2023graph}, our objective does not involve balancing class sizes as this loss affects the attention weights rather than the predictions.
Specifically, we employ L2-norm of the summation of $\mathbf{s}_{t}^{(k)}(i,j)$ across $\{v_{i},v_{j}\} \in \mathcal{E}$.
Formally,
\begin{equation}\label{eq:reg}
     \mathcal{L}_{reg} =\frac{1}{K}\sum_{k=1}^{K} \frac{\sqrt{\mathcal{C}}}{|\mathcal{E}|} \Vert \sum_{\{v_i,v_j\} \in \mathcal{E}} \mathbf{s}^{(k)}(i,j) \Vert_2^2 - 1,
\end{equation}
where $K$ is the number of layers of \method.
Note that the value of $\mathcal{L}_{reg}$ decreases as the distribution of $\mathbf{s}^{(k)}(i,j)$ varies across $\forall \{v_{i},v_{j}\}\in\mathcal{E}$.

\textbf{Training \method.}
Finally, the parameters of $K$-layer \method (spec., initial feature projector \texttt{MLP}, projection- and attention weights of each layer $\mathbf{W}^{(k)},\mathbf{W}_{att}^{(k)},\forall k \in [K]$) are trained via (1) the regularization loss $\mathcal{L}_{reg}$ (Eq.~\eqref{eq:reg}) and (2) the loss designed for the particular learning task (e.g., node classification), which we denote as $\mathcal{L}_{task}$.
Specifically, we use the loss function: $\mathcal{L} = \mathcal{L}_{task} + \lambda \mathcal{L}_{reg}$, where $\lambda$ is a hyperparameter that controls the strength of the regularization term $\mathcal{L}_{reg}$.

The time complexity of \method is analyzed in Appendix~\ref{app:time_complexity}, showing that it exhibits a time complexity similar to that of GCN~\citep{kipf2016semi}.

}

\section{Experiments}{
\label{sec:experiment}
\begin{table}[t]
\scriptsize
\setlength\tabcolsep{1.5pt}
    \centering
    \begin{tabular}{c c c c c c}
    \toprule
       &\textbf{Edge Hom.} &\textbf{\# Nodes} &\textbf{\# Edges} &\textbf{\# Features} &\textbf{\# Classes} \\
       \midrule
    \textbf{Texas}  &0.21 &183 &295 &1,703 &5 \\
    \textbf{Wisconsin} & 0.11 &251 &466 &1,703 &5 \\
    \textbf{Cornell} &0.30 &183 &280 &1,703 &5 \\
    \textbf{Actor} &0.22 &7,600 &26,752 &931 &5\\
    \textbf{Squirrel} & 0.22 &5,201 &198,493 &2,089 &5\\
    \textbf{Chameleon} &0.23 &2,277 &31,421 &2,325 &5 \\
    \textbf{Cora} &0.81 &2,708 &5,278 &1,433 &7 \\
    \textbf{Citeseer} &0.74 &3,327 &4,676 &3,703 &6 \\
    \textbf{Pubmed} &0.80 &19,717 &44,327 &500 &3 \\
    \textbf{Penn94} &0.47 &41,554 &1,362,229 &5 &2\\
    \textbf{Amazon-rat.} &0.38 &24,492 &93,050 &300 &5\\
    \bottomrule
    \end{tabular}
    \caption{Dataset statistics. The edge homophily ratio (Edge Hom.) is defined as $|\{(v_i,v_j) \in \mathcal{E} \}: y_i=y_j |/|\mathcal{E}|$.
}
    \label{tab:dataset_stat}
\end{table}

\begin{table*}[htp]
\scriptsize
\setlength\tabcolsep{2.5pt}
\renewcommand{\arraystretch}{0.95}
    \centering
    \begin{tabular}{c c c c c c c c c c c c c}
   \toprule
       &\textbf{Texas} &\textbf{Wiscon.} &\textbf{Cornell} &\textbf{Actor} &\textbf{Squir.} &\textbf{Chamel.}   &\textbf{Cora} &\textbf{Citeseer} &\textbf{Pubmed} &\textbf{Penn94} &\textbf{Amazon-rat.} &\textbf{A.R.} \\
        
    \midrule
    MLP & 80.81 \scriptsize{± 4.7} & 85.29 \scriptsize{± 3.3} &81.89 \scriptsize{± 6.4} &36.53 \scriptsize{± 0.7} &28.77 \scriptsize{± 1.5} &46.21 \scriptsize{± 2.9} &75.69 \scriptsize{± 2.0} &74.02 \scriptsize{± 1.9} &87.16 \scriptsize{± 0.3} &73.61 \scriptsize{± 0.4} &42.87 \scriptsize {± 0.4} &9.9\\
    \midrule
    GCN &55.14 \scriptsize{± 5.1} &51.76 \scriptsize{± 3.0} &60.54 \scriptsize{± 5.3} &27.32 \scriptsize{± 1.1} &53.43 \scriptsize{± 2.0} &64.82 \scriptsize{± 2.2} &86.98 \scriptsize{± 1.2} &76.50 \scriptsize{± 1.3} &88.42 \scriptsize{± 0.5} &82.47 \scriptsize{± 0.2} &48.70 \scriptsize{± 0.6} &9.1\\
    GAT &52.16 \scriptsize{± 6.6} &49.41 \scriptsize{± 4.0} &61.89 \scriptsize{± 5.0} &27.44 \scriptsize{± 0.8} &40.72 \scriptsize{± 1.5} &60.26 \scriptsize{± 2.5} &87.30 \scriptsize{± 1.1} &76.55 \scriptsize{± 1.2} &86.33 \scriptsize{± 0.4} &81.53 \scriptsize{± 0.5} & \secb 49.09 \scriptsize{± 0.6} &9.9\\
    \midrule
    GPR-GNN &78.38 \scriptsize{± 4.3} &82.94 \scriptsize{± 4.2} &80.27 \scriptsize{± 8.1} &34.63 \scriptsize{± 1.2} &31.61 \scriptsize{± 1.2} &46.58 \scriptsize{± 1.7} &87.95 \scriptsize{± 1.1} &77.13 \scriptsize{± 1.6} &87.54 \scriptsize{± 0.3} &81.38 \scriptsize{± 0.1} &44.88 \scriptsize{± 0.3} &8.6\\
    FAGCN  &77.00 \scriptsize {± 7.7} &78.32 \scriptsize {± 6.3} &82.41 \scriptsize{± 3.8} &35.67 \scriptsize{± 0.9}  &42.20 \scriptsize{± 1.8}  &60.98 \scriptsize{± 2.3} &87.42 \scriptsize{± 2.1} &76.35 \scriptsize{± 1.7} &87.83 \scriptsize{± 1.1} &72.85 \scriptsize{± 0.5} & 44.12 ± \scriptsize{0.3} &9.1\\

    GGCN &84.86 \scriptsize{± 4.5} &86.86 \scriptsize{± 3.2} &85.68 \scriptsize{± 6.6} &\secb 37.54 \scriptsize{± 1.5} &55.17 \scriptsize{± 1.5} &71.14 \scriptsize{± 1.8} &87.95 \scriptsize{± 1.0} &77.14 \scriptsize{± 1.4} &89.15 \scriptsize{± 0.3} &OOM  &36.86 \scriptsize{± 0.4} &6.0\\
    
    ACM-GCN & \secb 87.84 \scriptsize{± 4.4} & \secb 88.43 \scriptsize{± 3.2} &85.14 \scriptsize{± 6.0} &36.28 \scriptsize{± 1.0} &54.40 \scriptsize{± 1.8} &66.93 \scriptsize{± 1.8} &87.91 \scriptsize{± 0.9} & \best 77.32 \scriptsize{± 1.7} &\thib 90.00 \scriptsize{± 0.5} &82.52 \scriptsize{± 0.9} &\thib 38.62 \scriptsize{± 0.6} &4.6 \thib\\
    Goal &83.62 \scriptsize{± 6.7} &86.98 \scriptsize{± 4.4} &\thib 85.68 \scriptsize {± 6.2} &36.46 \scriptsize{± 1.0} &60.53 \scriptsize{± 1.6} 
    &\thib 71.65 \scriptsize{± 1.6} & \best 88.75 \scriptsize{± 0.8} &77.15 \scriptsize{± 0.9} &89.25 \scriptsize{± 0.5} &\secb 84.18 \scriptsize{± 0.3} &37.94 \scriptsize{± 0.3}  &4.7\\

    AERO-GNN &84.35 \scriptsize{± 5.2} &81.24 \scriptsize{± 6.8} &84.80 \scriptsize{± 3.3} &36.57 \scriptsize{± 1.1}  & \thib 61.76 \scriptsize{± 2.4} &71.58 \scriptsize{± 2.4} &88.12 \scriptsize{± 1.1} &77.08 \scriptsize{± 1.5} &89.95 \scriptsize{± 0.7} &82.47 \scriptsize{± 0.7} &45.71 \scriptsize{± 0.5} &5.1\\
    \midrule
    
    Ord. GNN &86.22 \thib \scriptsize{± 4.1} & \thib 88.04 \scriptsize{± 3.6} &\best 87.03 \scriptsize{± 4.7} & \best 37.99 \scriptsize{± 1.0} & \secb 62.44 \scriptsize{± 1.9}
    &\secb 72.28 \scriptsize{± 2.2} & \secb 88.37 \scriptsize{± 0.7}& \secb 77.31 \scriptsize{± 1.7} & \secb 90.15 \scriptsize{± 0.3} &\thib 83.65 \scriptsize{± 0.6} &38.52 \scriptsize{± 0.4} &\secb 2.7\\
    
    H\textsubscript{2}GCN  &84.86 \scriptsize{± 7.2} &87.65 \scriptsize{± 4.9} &82.70 \scriptsize{± 5.2} &35.70 \scriptsize{± 1.0} &36.48 \scriptsize{± 1.8} &60.11 \scriptsize{± 2.1} &87.87 \scriptsize{± 1.2} &77.11 \scriptsize{± 1.5} &89.49 \scriptsize{± 0.3} &81.31 \scriptsize{± 0.6} &36.47 \scriptsize{± 0.2} & 7.8\\
    
    \midrule
    DMP &66.08 \scriptsize{± 7.0} &56.41 \scriptsize{± 5.5} &62.73 \scriptsize{± 4.5} &28.30 \scriptsize{± 2.7} &34.19 \scriptsize{± 7.6}  &63.79 \scriptsize{± 4.1} &82.56  \scriptsize{± 1.9} &62.54 \scriptsize{± 1.5} &73.12 \scriptsize{± 0.9} & 73.85 \scriptsize{± 0.7} &35.84 \scriptsize{± 0.4} &11.2\\
    \midrule
    
    \textbf{\method} & \best 89.19 \scriptsize{± 4.5} & \best 89.01 \scriptsize{± 4.1}  & \secb 86.48 \scriptsize{± 6.1} & \thib 36.72 \scriptsize{± 1.6} & \best 63.60 \scriptsize{± 1.7} & \best 75.20 \scriptsize{± 2.3}  &\thib 88.12 \scriptsize{± 1.0} & \thib 77.20 \scriptsize{± 1.8}  & \best 90.35 \scriptsize{± 0.6} & \best 85.94 \scriptsize{± 0.4} & \best 49.18 \scriptsize{± 0.6} & \best 1.7
    \\
    \bottomrule
    \end{tabular}
    \caption{Node classification accuracy (\%) on $11$ datasets. 
    The \textcolor{dodgerblue}{best}, \textcolor{secgreen}{second-best}, and \textcolor{yellow}{third-best} performance across each dataset are highlighted in \textcolor{dodgerblue}{blue}, \textcolor{secgreen}{green}, and \textcolor{yellow}{yellow}, respectively. 
    A.R. and OOM denote average ranking and out-of-memory, respectively.
    Overall, our method \textbf{M2M-GNN} performs the best in terms of average ranking, achieving at least the third-best performance in every dataset.}
    \label{tab:result}
\end{table*}

In this section, we conduct a comprehensive evaluation of \method on various benchmark datasets, including both homophilic and heterophilic graphs. We aim to answer four key research questions: (\textbf{RQ1}) How does \method perform in node classification tasks? (\textbf{RQ2}) What patterns emerge in the attention scores, $\mathbf{s}^{(k)}(i, j)$'s? (\textbf{RQ3}) How effectively does \method mitigate oversmoothing? (\textbf{RQ4}) Are all components of \method necessary and impactful?


\paragraph{Datasets.} 
We use 11 widely-used node classification benchmark datasets, where 8 are heterophilic (Texas, Wisconsin, Cornell, Actor, Squirrel, Chameleon, Penn94, and Amazon-rating), and the remaining 3 are homophilic (Cora, Citeseer, and Pubmed).
In line with prior research, for training/validation/test splits, we employ the setting provided by Pytorch Geometric (PyG)~\citep{platonov2022critical}.
For details of the datasets, including their sources and construction methods, refer to Appendix~\ref{dataset_and_baseline}.


\vspace{-3mm}

\paragraph{Baseline models.}
We compare the performance of \method against the following $12$ baseline models:
\begin{enumerate}[leftmargin=*,noitemsep,topsep=0pt] 
    \item \textbf{Classic models}: GCN~\citep{kipf2016semi}, GAT~\citep{velivckovic2018graph}, and MLP,
    \item \textbf{SMP-based GNNs}: ACM-GCN~\citep{luan2022revisiting}, GPRGNN~\citep{chien2020adaptive}, FAGCN~\citep{bo2021beyond}, GGCN~\citep{yan2022two},  Goal~\citep{DBLP:conf/icml/ZhengZLZWP23}, and AERO-GNN~\cite{DBLP:conf/icml/LeeBYS23},
    \item \textbf{Concatenation-based GNNs}: Ordered GNN~\citep{song2022ordered} and H\textsubscript{2}GCN~\citep{zhu2020beyond},
    \item \textbf{Attention-vector-based GNN}: DMP~\citep{yang2021diverse}.
\end{enumerate}
For details of them,
refer to Appendix~\ref{dataset_and_baseline}. 




\subsection{\textbf{RQ1}: Node Classification Performance Results}
In Table \ref{tab:result}, we provide the mean node classification accuracy for test nodes, along with their corresponding standard deviations, across 10 random data splits.

In summary, \method achieves the best average ranking among all the methods, securing a position in the top 3 across all datasets.
The following three points stand out.

First, \method outperforms all SMP-based GNNs. This observation serves as empirical support for our theoretical finding, indicating that the limitations of SMP (Sec.~\ref{sec:limitation}) can adversely impact its effectiveness. 
Conversely, the theoretical properties of \textit{m-2-m} (Sec.~\ref{subsec:m2mgnn_motivation}) contribute to the enhanced classification accuracy of \method.

Second, \method also outperforms DMP, an attention-vector-based GNN: This again demonstrates the superiority of \textit{m-2-m} over \textit{m-2-e}.
DMP is a variant of SMP assigning distinct coefficients to each channel of node representations. Consequently, while employing weight vectors, DMP still follows the \textit{m-2-e} paradigm in its message passing scheme. 

Lastly, \method outperforms concatenation-based GNNs. While H\textsubscript{2}GCN and Ordered GNN employ concatenation to combine representations from different hops, they fail to address the mixing of features of heterophilic neighbors within the same hop. This validates the effectiveness of the novel concatenation design used in \textit{m-2-m}.

\subsection{RQ2: Attention Weight Analysis}\label{subsec:exp_visual}
We analyze the distribution of learned attention weights $\mathbf{s}^{(k)}(i,j)$'s. To this end, we consider directed graphs. In the case of an undirected graph, we convert it into a directed graph by splitting each undirected edge into two directed edges. For an edge from $v_i$ to $v_j$, $v_i$ and $v_j$ are referred to as the source and target nodes of the edge, respectively.

As described in Section~\ref{subsec:proposed_method}, the learned attention weight $\mathbf{s}^{(k)}(i,j)$ corresponds to the soft label of target node $v_j$.
Let $\bar {\mathbf S}\in \mathbb R^{|\mathcal{E}|\times \mathcal{C}}$ be the average value across all layers, calculated as $\frac{1}{K} \sum_{k=1}^K \mathbf S^{(k)}$, where $\mathbf S^{(k)}$ is the matrix formed by all $\mathbf{s}^{(k)}(i,j) \in \mathbb R^{\mathcal{C}}$ $,\{v_i,v_j\} \in \mathcal{E}$.
To make a comparison, we also consider the ground truth edge labels $\hat{\mathbf{ S}}\in \mathbb R^{|\mathcal{E}|\times \mathcal{C}}$, where $\hat{\mathbf{ S}}_{ic}=1$ if the target node of the $i$-th edge belongs to class $c$. 
A learned $\bar{\mathbf{ S}}$ is considered accurate if, upon proper column index reordering, it is similar to $\hat{\mathbf{ S}}$. 
We reorder the columns of $\mathbf{\bar{S}}$ to ensure their values are properly aligned to the order of class.
Then we compute the matrix $\mathbf{\mathcal{S}}=\hat{\mathbf{ S}}^T \hat{\mathbf{ S}} \in \mathbb R^{\mathcal{C}\times \mathcal{C}}$, followed by a row-wise softmax. 

In the definition of $\mathbf{\mathcal{S}}$, large diagonal entries imply predictions closely align with the ground truth, thus considered correct.
As shown in Fig.~\ref{fig:visualization}, the learned $\bar{\mathbf{S}}$ and the ground truth $\hat{\mathbf{S}}$ are well matched. 
Specifically, 19 out of the 23 columns (each column corresponds to one class, e.g., Cora has 6 classes of nodes) of $\mathcal{S}$ are verified to be correct, (i.e., the diagonal entries are the largest in the columns), with a particularly high accuracy observed for the homophilic graphs where nearly all edges are correctly identified.





\begin{figure}[t]
  \centering
    \includegraphics[width=\columnwidth]
    {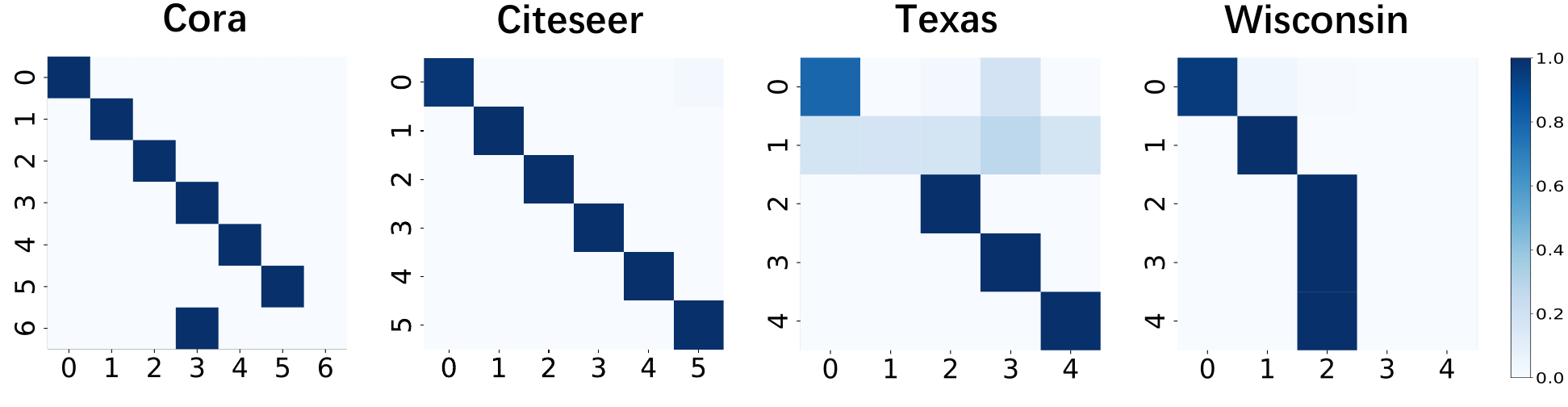}
    \caption{
    Visualization results of $\mathbf{\mathcal{S}}$ for two homophilic (left two) and two heterophilic (right two) graphs.
    (Off-)Diagonal entries indicate (in)correct prediction.} 
    \label{fig:visualization}
\end{figure}

\subsection{RQ3: Over-smoothing Analysis}

We evaluate the robustness of \method against the over-smoothing problem by assessing its performance at different model depths ($K \in \{2,4,8,16,32,64\}$). 
To provide a comparative analysis against other GNNs, we also evaluate three SMP-based GNNs: FAGCN, Goal, and GPR-GNN.
We evaluate the performances of these methods on two heterophilic graphs and one homophilic graph.


The analysis presented in Fig.~\ref{fig:os_ana} illustrates that the performance of \method remains relatively stable with varying numbers of layers, achieving its best performance when deeper layers are employed ($K=32$) in the Cora dataset. In contrast, the SMP-based GNNs utilized in the experiment exhibit a substantial decrease in performance as the number of layers increases, indicating their empirical susceptibility to the oversmoothing problem.

\begin{figure}[t]
    \centering
    \includegraphics[width=\columnwidth]{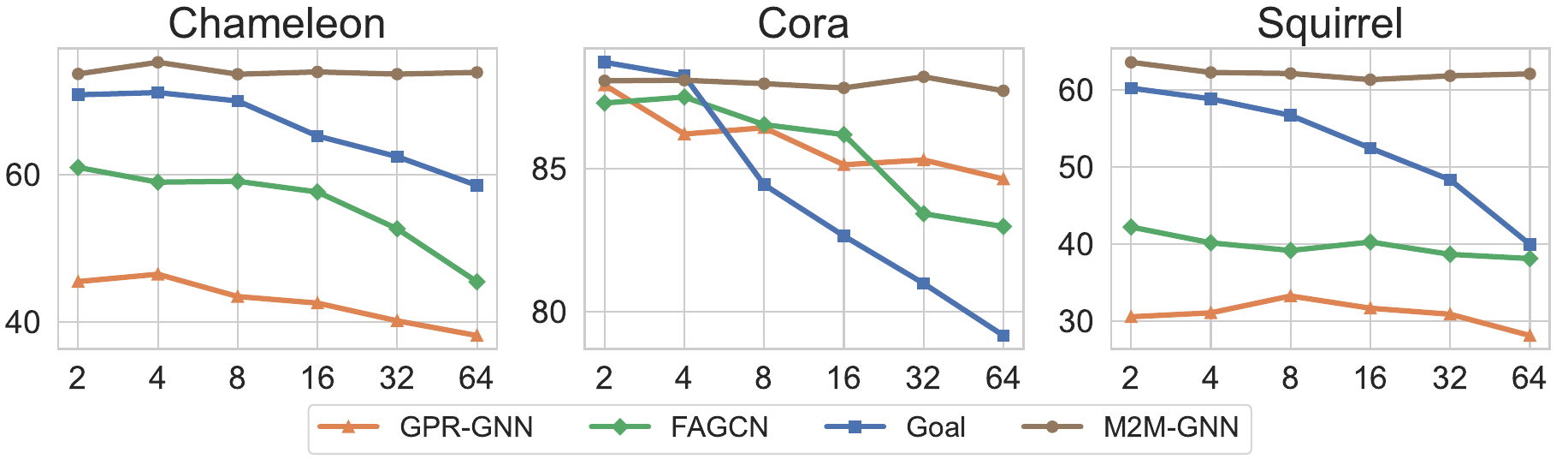}
    \caption{Performance comparison of \method against SMP-based GNNs under various model depths.
    The X-axis has the number of layers, and the Y-axis has node classification accuracy.
    }
    \label{fig:os_ana}
\end{figure}

\color{black}
\subsection{RQ4: \method Component Analysis}
We examine the following: (1) the necessity of regularization loss $\mathcal{L}_{reg}$ (see Eq.~\eqref{eq:reg}) and (2) the impact of the number of chunks $\mathcal{C}$.
Specifically, we assess the performance of \method with/without $\mathcal{L}_{reg}$ and across various values of $\mathcal{C}$.
We focus on evaluating (1) the discriminative power of soft labels, (2) node classification accuracy, and (3) robustness against oversmoothing.

To this end, we use three metrics: 
(1) mixing score, which is defined as the proportion of heterophilic edges $\{v_{i},v_{j}\}$ such that the argmax chunk is different for $(v_{i} \rightarrow v_{j})-$ and $(v_{j} \rightarrow v_{i})-$message passing,
(2) node classification accuracy at the best hyperparameter setting for each, and
(3) node classification accuracy at $K=32$.
Across all metrics, a higher value indicates better model performance.
We utilized one heterophilic dataset (Wisconsin) and one homophilic dataset (Pubmed).

We present results in Fig.~\ref{fig:ablation}, where
each tuple on the X-axis represents a pair of parameters ($\mathcal{C}, \lambda$). Recall that $\lambda$ is a scalar for $\mathcal{L}_{reg}$, and $\lambda = 0$ indicates the absence of the regularization loss.
First, $\mathcal{L}_{reg}$ is demonstrated to be necessary, given that the models with $\mathcal{L}_{reg} \neq 0$ outperforms that with $\mathcal{L}_{reg} = 0$ across all metrics and datasets.
Second, setting the number of chunks and classes equal leads to the best performance. 
While increasing $\mathcal{C}$ may help discriminate heterophilic node representations, this may cause excessive separation, which is harmful to smoothing homophilic node representations, leading to suboptimal node embeddings. 
\color{black}



}

\section{Related Work}{
\label{sec:related_work}
\textbf{Analysis of signed message passing.}
The theoretical analysis by \citet{yan2022two} explores the impact of signed messages in GNNs under a binary-class setting, introducing the desirable signs for homophilic and heterophilic edges, respectively.
\citet{choi2023signed} extend the theoretical framework to a multi-class setting, analyzing how node embeddings are updated in multi-class \texttt{SMP}. 

\citet{choi2023signed} investigate the impact of the error rate on performance by assigning positive and negative coefficients to heterophilic and homophilic edges, respectively. However, they do not discuss the suboptimality of \texttt{SMP} for GNNs with multiple layers, even when the error rate is zero (i.e., desirable). 

\citet{NEURIPS2023_23aa2163} reveal that for a triad in a graph (i.e., three connected nodes), their embedding update can be undesirable even when they are connected by edges with proper signs. However, they overlook the influence of other nodes and do not present a solution to mitigate this limitation, instead sidestepping it by employing a one-layer GNN.



\textbf{Oversmoothing analysis.}
Among many studies of oversmoothing~\citep{oono2019graph,zhou2021dirichlet, keriven2022not, wu2023demystifying, wu2022non, DBLP:conf/icml/LeeBYS23, DBLP:journals/tkde/LiangXSKQY24}, we focus on two most related ones.
\citet{wu2022non} present a non-asymptotic analysis of the oversmoothing problem on binary-class graphs, focusing on positive propagation coefficient settings. 
\citet{DBLP:conf/icml/LeeBYS23} study oversmoothing phenomena in attention-based methods under asymptotic cases where the number of model layers goes infinity.

Unlike \citep{wu2022non}, we focus on multi-class settings with signed message passing, and compared to \citep{DBLP:conf/icml/LeeBYS23}, we provide non-asymptotic results.

\textbf{Concatenation-based GNNs.}
To avoid undesirable feature smoothing, many concatenation-based GNNs have been proposed~\citep{hamilton2017inductive, xu2018representation, zhu2020beyond,  song2022ordered}.
They mainly (1) concatenate ego node features and neighbor features and/or (2) concatenate node features at different distances.

Note that LW-GCN~\citep{dai2022label} also focuses on separating node features from various classes. However, the authors do not investigate the constraints of SMP or establish theoretical properties for addressing these constraints. 
In contrast to LW-GCN's goal of accurately mapping nodes to chunks according to their labels, M2M-GNN adopts a more practical strategy by assigning nodes with different labels to separate chunks.
}

\begin{figure}[t]
    \centering
    \includegraphics[width=1.0\columnwidth]{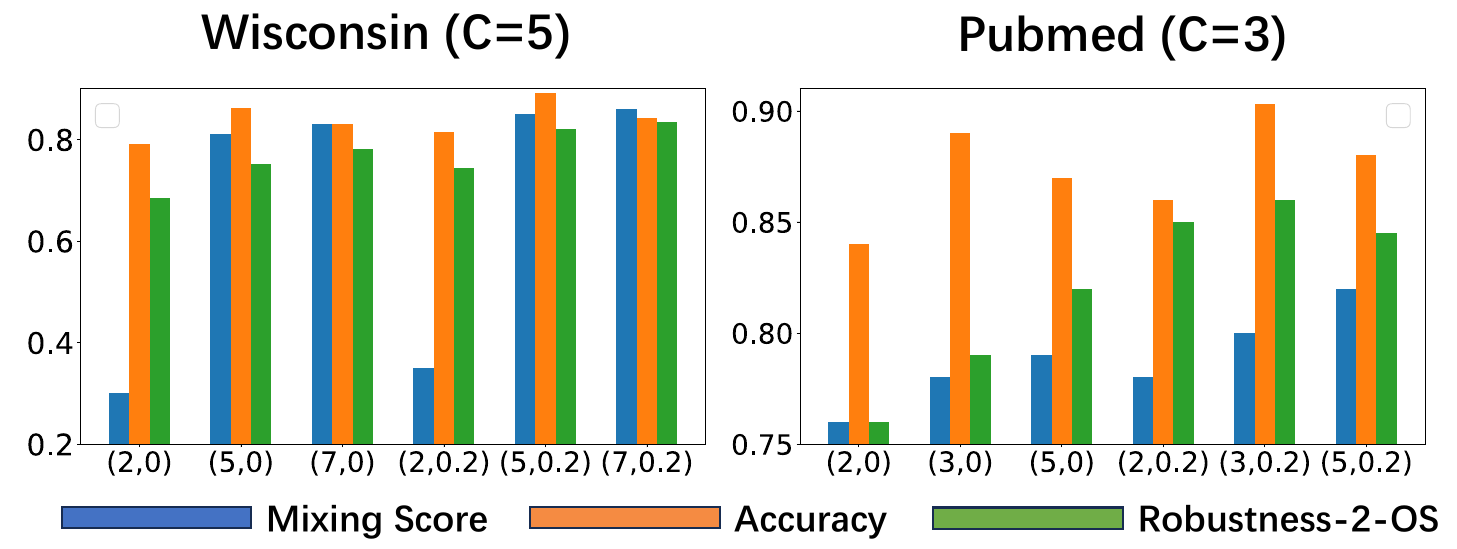}
    \caption{Ablation study. Each tuple on the X-axis indicates a pair of parameters: (number of chunks $\mathcal{C}$, strength of regularization $\lambda$).
    For all the metrics, higher values indicate better performance. \label{fig:ablation}}
\end{figure}

\vspace{-2mm}

\section{Conclusion}{
\label{sec:conclusion}
In this paper, we reveal two unexpected challenges associated with signed message passing (SMP): (a) undesirable embedding update and (b) vulnerability to oversmoothing.
To mitigate these challenges, we develop a novel multiset-to-multiset (m-2-m) message passing scheme and prove several desirable theoretical properties of it.
Motivated by the properties, we propose \method, a novel GNN based on the m-2-m scheme.
Through our comprehensive experiments,
we demonstrate the superiority of \method over existing SMP-based GNNs in node classification tasks.

}

\section*{Acknowledgements}
We want to thank all the anonymous reviewers for their constructive comments. The authors also thank Fanchen Bu for his insightful feedback on Section 3.
This work was partially supported by the Major Key Project of PCL
(No. 2022ZD0115301), and an Open Research Project of Zhejiang Lab (NO.2022RC0AB04).

\section*{Impact Statement}{
\label{sec:broader_impact}
This paper presents work whose goal is to advance graph representation learning with heterophily, a prevalent trend in real-world connections observed in domains like chemical compounds and social networks. 

The effectiveness of our model may be compromised by inaccurate node labels, as labeled nodes are crucial for providing supervision signals during model training, potentially leading to adverse implications. For instance, utilizing a misled model to detect fraud within a trading network could introduce misclassifications and bias against legitimate users. To mitigate the vulnerability to noisy labels, established techniques can be integrated into our methodology. Additionally, it is essential to formulate specific ethical guidelines and impose constraints on the application of our model based on practical considerations.

Our model generates node representations using node attributes and linear transformation. Within the context of a social network represented as a graph, these attributes may encompass sensitive personal data like age and gender. To improve fairness, it may be helpful to incorporate the notion of counterfactual fairness, which emphasizes that a decision made about an individual is considered fair if altering the individual's sensitive attribute does not influence the decision outcome. 

Unlike graph attention networks, which explicitly learn attention weights for each neighbor, our model makes it challenging to quantify the individual impact of each neighbor. Consequently, the interpretability of predictions in our model is not as straightforward as in the graph attention network. In order to enhance the transparency of predictions, existing graph explainers can be employed. For example, when focusing on a specific node $v_i$, masking individual edges allows us to pinpoint which edge exerts the most significant influence on its prediction. Analogous techniques can be applied to a particular node or attribute.
}

\clearpage

\bibliography{paperbib}

\begin{thebibliography}{44}
\providecommand{\natexlab}[1]{#1}
\providecommand{\url}[1]{\texttt{#1}}
\expandafter\ifx\csname urlstyle\endcsname\relax
  \providecommand{\doi}[1]{doi: #1}\else
  \providecommand{\doi}{doi: \begingroup \urlstyle{rm}\Url}\fi

\bibitem[Atay \& Tuncel(2014)Atay and Tuncel]{atay2014spectrum}
Atay, F.~M. and Tuncel, H.
\newblock On the spectrum of the normalized laplacian for signed graphs: Interlacing, contraction, and replication.
\newblock \emph{Linear Algebra and its Applications}, 442:\penalty0 165--177, 2014.

\bibitem[Ba et~al.(2016)Ba, Kiros, and Hinton]{ba2016layer}
Ba, J.~L., Kiros, J.~R., and Hinton, G.~E.
\newblock Layer normalization.
\newblock \emph{arXiv preprint arXiv:1607.06450}, 2016.

\bibitem[Bo et~al.(2021)Bo, Wang, Shi, and Shen]{bo2021beyond}
Bo, D., Wang, X., Shi, C., and Shen, H.
\newblock Beyond low-frequency information in graph convolutional networks.
\newblock In \emph{Proceedings of the AAAI Conference on Artificial Intelligence}, volume~35, pp.\  3950--3957, 2021.

\bibitem[Brody et~al.(2021)Brody, Alon, and Yahav]{brody2021attentive}
Brody, S., Alon, U., and Yahav, E.
\newblock How attentive are graph attention networks?
\newblock In \emph{International Conference on Learning Representations}, 2021.

\bibitem[Chien et~al.(2020)Chien, Peng, Li, and Milenkovic]{chien2020adaptive}
Chien, E., Peng, J., Li, P., and Milenkovic, O.
\newblock Adaptive universal generalized pagerank graph neural network.
\newblock In \emph{International Conference on Learning Representations}, 2020.

\bibitem[Choi et~al.(2023)Choi, Choi, Ko, and Kim]{choi2023signed}
Choi, Y., Choi, J., Ko, T., and Kim, C.-K.
\newblock Is signed message essential for graph neural networks?
\newblock \emph{arXiv preprint arXiv:2301.08918}, 2023.

\bibitem[Chung \& Lu(2006)Chung and Lu]{chung2006concentration}
Chung, F. and Lu, L.
\newblock Concentration inequalities and martingale inequalities: a survey.
\newblock \emph{Internet mathematics}, 3\penalty0 (1), 2006.

\bibitem[Dai et~al.(2022)Dai, Zhou, Guo, and Wang]{dai2022label}
Dai, E., Zhou, S., Guo, Z., and Wang, S.
\newblock Label-wise graph convolutional network for heterophilic graphs.
\newblock In \emph{Learning on Graphs Conference}, pp.\  26--1. PMLR, 2022.

\bibitem[Defferrard et~al.(2016)Defferrard, Bresson, and Vandergheynst]{defferrard2016convolutional}
Defferrard, M., Bresson, X., and Vandergheynst, P.
\newblock Convolutional neural networks on graphs with fast localized spectral filtering.
\newblock \emph{Advances in neural information processing systems}, 29, 2016.

\bibitem[Deshpande et~al.(2018)Deshpande, Sen, Montanari, and Mossel]{deshpande2018contextual}
Deshpande, Y., Sen, S., Montanari, A., and Mossel, E.
\newblock Contextual stochastic block models.
\newblock \emph{Advances in Neural Information Processing Systems}, 31, 2018.

\bibitem[Fey \& Lenssen(2019)Fey and Lenssen]{fey2019fast}
Fey, M. and Lenssen, J.~E.
\newblock Fast graph representation learning with pytorch geometric.
\newblock \emph{arXiv preprint arXiv:1903.02428}, 2019.

\bibitem[Hamilton et~al.(2017)Hamilton, Ying, and Leskovec]{hamilton2017inductive}
Hamilton, W., Ying, Z., and Leskovec, J.
\newblock Inductive representation learning on large graphs.
\newblock \emph{Advances in neural information processing systems}, 30:\penalty0 1024--1034, 2017.

\bibitem[Keriven(2022)]{keriven2022not}
Keriven, N.
\newblock Not too little, not too much: a theoretical analysis of graph (over) smoothing.
\newblock \emph{Advances in Neural Information Processing Systems}, 35:\penalty0 2268--2281, 2022.

\bibitem[Kingma \& Ba(2014)Kingma and Ba]{kingma2014adam}
Kingma, D.~P. and Ba, J.
\newblock Adam: A method for stochastic optimization.
\newblock \emph{arXiv preprint arXiv:1412.6980}, 2014.

\bibitem[Kipf \& Welling(2016)Kipf and Welling]{kipf2016semi}
Kipf, T.~N. and Welling, M.
\newblock Semi-supervised classification with graph convolutional networks.
\newblock In \emph{International Conference on Learning Representations}, 2016.

\bibitem[Lee et~al.(2023)Lee, Bu, Yoo, and Shin]{DBLP:conf/icml/LeeBYS23}
Lee, S.~Y., Bu, F., Yoo, J., and Shin, K.
\newblock Towards deep attention in graph neural networks: Problems and remedies.
\newblock In \emph{International Conference on Machine Learning}, volume 202, pp.\  18774--18795. {PMLR}, 2023.

\bibitem[Lee et~al.(2024)Lee, Kim, Bu, Yoo, Tang, and Shin]{DBLP:journals/corr/abs-2402-04621}
Lee, S.~Y., Kim, S., Bu, F., Yoo, J., Tang, J., and Shin, K.
\newblock Feature distribution on graph topology mediates the effect of graph convolution: Homophily perspective.
\newblock \emph{CoRR}, abs/2402.04621, 2024.

\bibitem[Li et~al.(2018)Li, Han, and Wu]{li2018deeper}
Li, Q., Han, Z., and Wu, X.-M.
\newblock Deeper insights into graph convolutional networks for semi-supervised learning.
\newblock In \emph{Proceedings of the AAAI conference on artificial intelligence}, volume~32, 2018.

\bibitem[Liang et~al.(2023)Liang, Hu, Xu, Song, and King]{NEURIPS2023_23aa2163}
Liang, L., Hu, X., Xu, Z., Song, Z., and King, I.
\newblock Predicting global label relationship matrix for graph neural networks under heterophily.
\newblock \emph{Advances in Neural Information Processing Systems}, 36:\penalty0 10909--10921, 2023.

\bibitem[Liang et~al.(2024)Liang, Xu, Song, King, Qi, and Ye]{DBLP:journals/tkde/LiangXSKQY24}
Liang, L., Xu, Z., Song, Z., King, I., Qi, Y., and Ye, J.
\newblock Tackling long-tailed distribution issue in graph neural networks via normalization.
\newblock \emph{{IEEE} Trans. Knowl. Data Eng.}, 36\penalty0 (5):\penalty0 2213--2223, 2024.

\bibitem[Lim et~al.(2021)Lim, Hohne, Li, Huang, Gupta, Bhalerao, and Lim]{lim2021large}
Lim, D., Hohne, F., Li, X., Huang, S.~L., Gupta, V., Bhalerao, O., and Lim, S.~N.
\newblock Large scale learning on non-homophilous graphs: New benchmarks and strong simple methods.
\newblock \emph{Advances in Neural Information Processing Systems}, 34:\penalty0 20887--20902, 2021.

\bibitem[Lin et~al.(2020)Lin, Quan, Wang, Ma, and Zeng]{lin2020kgnn}
Lin, X., Quan, Z., Wang, Z.-J., Ma, T., and Zeng, X.
\newblock Kgnn: Knowledge graph neural network for drug-drug interaction prediction.
\newblock In \emph{IJCAI}, pp.\  2739--2745, 2020.

\bibitem[Luan et~al.(2022)Luan, Hua, Lu, Zhu, Zhao, Zhang, Chang, and Precup]{luan2022revisiting}
Luan, S., Hua, C., Lu, Q., Zhu, J., Zhao, M., Zhang, S., Chang, X.-W., and Precup, D.
\newblock Revisiting heterophily for graph neural networks.
\newblock \emph{Advances in neural information processing systems}, 35:\penalty0 1362--1375, 2022.

\bibitem[Oliveira(2009)]{oliveira2009concentration}
Oliveira, R.~I.
\newblock Concentration of the adjacency matrix and of the laplacian in random graphs with independent edges.
\newblock \emph{arXiv preprint arXiv:0911.0600}, 2009.

\bibitem[Oono \& Suzuki(2019)Oono and Suzuki]{oono2019graph}
Oono, K. and Suzuki, T.
\newblock Graph neural networks exponentially lose expressive power for node classification.
\newblock In \emph{International Conference on Learning Representations}, 2019.

\bibitem[Paszke et~al.(2019)Paszke, Gross, Massa, Lerer, Bradbury, Chanan, Killeen, Lin, Gimelshein, Antiga, et~al.]{paszke2019pytorch}
Paszke, A., Gross, S., Massa, F., Lerer, A., Bradbury, J., Chanan, G., Killeen, T., Lin, Z., Gimelshein, N., Antiga, L., et~al.
\newblock Pytorch: An imperative style, high-performance deep learning library.
\newblock \emph{Advances in neural information processing systems}, 32, 2019.

\bibitem[Pei et~al.(2020)Pei, Wei, Chang, Lei, and Yang]{pei2020geom}
Pei, H., Wei, B., Chang, K. C.-C., Lei, Y., and Yang, B.
\newblock Geom-gcn: Geometric graph convolutional networks.
\newblock \emph{arXiv preprint arXiv:2002.05287}, 2020.

\bibitem[Platonov et~al.(2022)Platonov, Kuznedelev, Diskin, Babenko, and Prokhorenkova]{platonov2022critical}
Platonov, O., Kuznedelev, D., Diskin, M., Babenko, A., and Prokhorenkova, L.
\newblock A critical look at the evaluation of gnns under heterophily: Are we really making progress?
\newblock In \emph{International Conference on Learning Representations}, 2022.

\bibitem[Rozemberczki et~al.(2021)Rozemberczki, Allen, and Sarkar]{rozemberczki2021multi}
Rozemberczki, B., Allen, C., and Sarkar, R.
\newblock Multi-scale attributed node embedding.
\newblock \emph{Journal of Complex Networks}, 9\penalty0 (2):\penalty0 cnab014, 2021.

\bibitem[Song et~al.(2022)Song, Zhou, Wang, and Lin]{song2022ordered}
Song, Y., Zhou, C., Wang, X., and Lin, Z.
\newblock Ordered gnn: Ordering message passing to deal with heterophily and over-smoothing.
\newblock In \emph{International Conference on Learning Representations}, 2022.

\bibitem[Tang et~al.(2013)Tang, Hu, Gao, and Liu]{DBLP:conf/ijcai/TangHGL13}
Tang, J., Hu, X., Gao, H., and Liu, H.
\newblock Exploiting local and global social context for recommendation.
\newblock In \emph{{IJCAI}}, pp.\  2712--2718, 2013.

\bibitem[Tsitsulin et~al.(2023)Tsitsulin, Palowitch, Perozzi, and M{\"u}ller]{tsitsulin2023graph}
Tsitsulin, A., Palowitch, J., Perozzi, B., and M{\"u}ller, E.
\newblock Graph clustering with graph neural networks.
\newblock \emph{Journal of Machine Learning Research}, 24\penalty0 (127):\penalty0 1--21, 2023.

\bibitem[Veli{\v{c}}kovi{\'c} et~al.(2018)Veli{\v{c}}kovi{\'c}, Cucurull, Casanova, Romero, Li{\`o}, and Bengio]{velivckovic2018graph}
Veli{\v{c}}kovi{\'c}, P., Cucurull, G., Casanova, A., Romero, A., Li{\`o}, P., and Bengio, Y.
\newblock Graph attention networks.
\newblock In \emph{International Conference on Learning Representations}, 2018.

\bibitem[Wu et~al.(2022)Wu, Chen, Wang, and Jadbabaie]{wu2022non}
Wu, X., Chen, Z., Wang, W.~W., and Jadbabaie, A.
\newblock A non-asymptotic analysis of oversmoothing in graph neural networks.
\newblock In \emph{International Conference on Learning Representations}, 2022.

\bibitem[Wu et~al.(2023)Wu, Ajorlou, Wu, and Jadbabaie]{wu2023demystifying}
Wu, X., Ajorlou, A., Wu, Z., and Jadbabaie, A.
\newblock Demystifying oversmoothing in attention-based graph neural networks.
\newblock \emph{arXiv preprint arXiv:2305.16102}, 2023.

\bibitem[Xu et~al.(2018{\natexlab{a}})Xu, Hu, Leskovec, and Jegelka]{xu2018powerful}
Xu, K., Hu, W., Leskovec, J., and Jegelka, S.
\newblock How powerful are graph neural networks?
\newblock In \emph{International Conference on Learning Representations}, 2018{\natexlab{a}}.

\bibitem[Xu et~al.(2018{\natexlab{b}})Xu, Li, Tian, Sonobe, Kawarabayashi, and Jegelka]{xu2018representation}
Xu, K., Li, C., Tian, Y., Sonobe, T., Kawarabayashi, K.-i., and Jegelka, S.
\newblock Representation learning on graphs with jumping knowledge networks.
\newblock In \emph{International conference on machine learning}, pp.\  5453--5462. PMLR, 2018{\natexlab{b}}.

\bibitem[Yan et~al.(2022)Yan, Hashemi, Swersky, Yang, and Koutra]{yan2022two}
Yan, Y., Hashemi, M., Swersky, K., Yang, Y., and Koutra, D.
\newblock Two sides of the same coin: Heterophily and oversmoothing in graph convolutional neural networks.
\newblock In \emph{2022 IEEE International Conference on Data Mining (ICDM)}, pp.\  1287--1292. IEEE, 2022.

\bibitem[Yang et~al.(2021)Yang, Li, Liu, Wang, Cao, Guo, et~al.]{yang2021diverse}
Yang, L., Li, M., Liu, L., Wang, C., Cao, X., Guo, Y., et~al.
\newblock Diverse message passing for attribute with heterophily.
\newblock \emph{Advances in Neural Information Processing Systems}, 34:\penalty0 4751--4763, 2021.

\bibitem[Yang et~al.(2016)Yang, Cohen, and Salakhutdinov]{DBLP:conf/icml/YangCS16}
Yang, Z., Cohen, W.~W., and Salakhutdinov, R.
\newblock Revisiting semi-supervised learning with graph embeddings.
\newblock In \emph{International Conference on Machine Learning}, volume~48, pp.\  40--48. PMLR, 2016.

\bibitem[Zheng et~al.(2022)Zheng, Liu, Pan, Zhang, Jin, and Yu]{DBLP:journals/corr/abs-2202-07082}
Zheng, X., Liu, Y., Pan, S., Zhang, M., Jin, D., and Yu, P.~S.
\newblock Graph neural networks for graphs with heterophily: {A} survey.
\newblock \emph{CoRR}, abs/2202.07082, 2022.

\bibitem[Zheng et~al.(2023)Zheng, Zhang, Lee, Zheng, Wang, and Pan]{DBLP:conf/icml/ZhengZLZWP23}
Zheng, Y., Zhang, H., Lee, V.~C., Zheng, Y., Wang, X., and Pan, S.
\newblock Finding the missing-half: Graph complementary learning for homophily-prone and heterophily-prone graphs.
\newblock In \emph{International Conference on Machine Learning}, volume 202, pp.\  42492--42505. {PMLR}, 2023.

\bibitem[Zhou et~al.(2021)Zhou, Huang, Zha, Chen, Li, Choi, and Hu]{zhou2021dirichlet}
Zhou, K., Huang, X., Zha, D., Chen, R., Li, L., Choi, S.-H., and Hu, X.
\newblock Dirichlet energy constrained learning for deep graph neural networks.
\newblock \emph{Advances in Neural Information Processing Systems}, 34:\penalty0 21834--21846, 2021.

\bibitem[Zhu et~al.(2020)Zhu, Yan, Zhao, Heimann, Akoglu, and Koutra]{zhu2020beyond}
Zhu, J., Yan, Y., Zhao, L., Heimann, M., Akoglu, L., and Koutra, D.
\newblock Beyond homophily in graph neural networks: Current limitations and effective designs.
\newblock \emph{Advances in neural information processing systems}, 33:\penalty0 7793--7804, 2020.

\end{thebibliography}
\bibliographystyle{icml2024}

\clearpage

\appendix
\onecolumn

\section{Discriminative Power of \textit{m-2-m}}\label{app:discriminative_power}
The work~\citep{xu2018powerful} characterizes the discriminative power of GNNs by analyzing whether a GNN maps two different neighborhoods (i.e., two multisets) to the same representation.

\textit{m-2-m} consistently yields higher discriminative power compared to \textit{m-2-e}. The intuition behind this is that distinguishing between two multisets is typically expected to be easier compared to distinguishing between two single elements. This can be understood similarly to classifying multi-dimensional features versus classifying one-dimensional features: additional dimensions provide more evidence for accurate classification. Formally, we present the following result.
\begin{lemma}\label{thm:app_dis_pow}
    Under the same setting as described in Theorem~\ref{theorem:discri_pow}: 
    
    (1) Any two distinct multisets that can be distinguished by \textit{m-2-e} ($\mathbf{m}_a^{m2e}\neq \mathbf{m}_b^{m2e}$) can also be discriminated by \textit{m-2-m} ($\mathbf{m}_a^{m2m}\neq \mathbf{m}_b^{m2m}$).
    
    (2) There exist two different multisets such that \textit{m-2-m} can differentiate, whereas \textit{m-2-e} is unable to discriminate them: $\mathbf{m}_a^{m2e}= \mathbf{m}_b^{m2e}$.
\end{lemma}
Over-smoothing arises if two distinct multisets are mapped to the same embedding. The above result also suggests that \textit{m-2-m} exhibits a higher level of robustness against over-smoothing compared to \textit{m-2-e}.

\section{Proof}\label{appendix:fullproof}
\subsection{Proof of Theorem \ref{thm:oversmoothing} and The Concentration Bound}\label{app:concentration}
\subsubsection{Proof of Theorem \ref{thm:oversmoothing}}
\begin{proof}
    Based on the definition of message passing scheme and CSBM defined in Sec. \ref{sec:limitation}, the expectation $\mathbb{E}_{\mathcal{G}}[\bar {\mathbf{\mathbf{h}}}_a^{(k)}]$ can be calculated as
\begin{equation}\label{eq:exp_mean_class_a}
     \mathbb{E}_{\mathcal{G}}[\bar {\mathbf{\mathbf{h}}}_a^{(k)}] = \frac{p \mathbb{E}_{\mathcal{G}}[ \bar {\mathbf{\mathbf{h}}}_a^{(k-1)}] - q\sum_{c\in[C], c \neq a}\mathbb{E}_{\mathcal{G}}[ \bar {\mathbf{\mathbf{h}}}_c^{(k-1)}]}{p+(C-1)q},
\end{equation}
and similarly, we have
\begin{equation}\label{eq:exp_mean_class_b}
    \mathbb{E}_{\mathcal{G}}[\bar {\mathbf{\mathbf{h}}}_b^{(k)}] = \frac{p \mathbb{E}_{\mathcal{G}}[ \bar {\mathbf{\mathbf{h}}}_b^{(k-1)}] - q\sum_{c\in[C], c \neq a}\mathbb{E}_{\mathcal{G}}[ \bar {\mathbf{\mathbf{h}}}_c^{(k-1)}]}{p+(C-1)q}.
\end{equation}
By subtracting $\mathbb{E}_{\mathcal{G}}[\bar {\mathbf{\mathbf{h}}}_b^{(k)}]$ from $\mathbb{E}_{\mathcal{G}}[\bar {\mathbf{\mathbf{h}}}_a^{(k)}]$, we get
\begin{equation}
    \mathbb{E}_{\mathcal{G}}[\bar {\mathbf{\mathbf{h}}}_a^{(k)} - \bar {\mathbf{\mathbf{h}}}_b^{(k)} ]= \frac{p+q}{(p+(C-1)q) } \mathbb{E}_{\mathcal{G}}[\bar {\mathbf{\mathbf{h}}}_a^{(k-1)} - \bar {\mathbf{\mathbf{h}}}_b^{(k-1)}].
\end{equation}
Therefore, the sequence $(\mathbb{E}_{\mathcal{G}}[\bar {\mathbf{\mathbf{h}}}_a^{(0)} - \bar {\mathbf{\mathbf{h}}}_b^{(0)} ],\ldots, \mathbb{E}_{\mathcal{G}}[\bar {\mathbf{\mathbf{h}}}_a^{(K)} - \bar {\mathbf{\mathbf{h}}}_b^{(K)} ]$ is a geometric sequence. We write
\begin{equation}
    \mathbb{E}_{\mathcal{G}}[\bar {\mathbf{\mathbf{h}}}_a^{(K)} - \bar {\mathbf{\mathbf{h}}}_b^{(K)} ] = (\frac{p+q}{p+(C -1)q})^K (\mathbf{u}_{a} - \mathbf{u}_{b}).
\end{equation}
\end{proof}
\subsubsection{Concentration inequality}
Let $\Vert\mathbf{M}\Vert$ denote the spectral norm of matrix $\mathbf{M}$.
Next, we derive a concentration bound for $\bar {\mathbf{\mathbf{h}}}_a^{(K)} - {\bar {\mathbf{h}}}_b^{(K)}$. We first present the Chernoff bound.
\begin{lemma}\label{lemma:1}
    \cite{chung2006concentration} Let $\mathbf{X}_{i} \sim Bern(p_i)$ be independent and $\mathbf Z=\sum_{i=1}^n \mathbf{X}_i$, $\bar {\mathbf{Z}} = \mathbb E(\mathbf{Z})$. Then for any $ \sigma>0
    $,
    \begin{equation}
        \mathbb{P}(\mathbf{Z} \ge (1+\sigma) \bar {\mathbf{Z}}) \le e^{-\frac{\sigma^2 }{2+\sigma}} \bar {\mathbf{Z}}
    \end{equation}
    \begin{equation}
        \mathbb{P}(\mathbf{Z} \le (1-\sigma) \bar {\mathbf{Z}}) \le e^{-\frac{\sigma^2 }{2}} \bar {\mathbf{Z}}
    \end{equation}
\end{lemma}
Now we are prepared to establish the concentration inequality for node degrees.
\begin{corollary} \label{corollary:1}
    For any $ \sigma>0$ and $r>0$, there exists a constant $\kappa(\sigma, r)$ such that when $\bar{d} \ge \kappa \log N$, the following holds with probability at least $1- N^{-r}$,
    \begin{equation}\nonumber
        (1-\sigma) \bar d \le d_i \le (1+\sigma) \bar d, \quad \text{for all} \; 1 \le i \le N.
    \end{equation}
\end{corollary}
\begin{proof}

For any $ i \in [N]$, $d_i=\sum_{j=1}^{N} \mathbf{X}_j$, where $ \mathbf{X}_j \sim Bern(p)$ if $v_i$ and $v_i$ share the same label and $\mathbf{X}_j \sim Bern(q)$ otherwise. Applying Lemma \ref{lemma:1}, we get
     \begin{equation}
         \mathbb P(d_i \le (1-\sigma)\bar d) \le e^{-\frac{\sigma^2}{2}\bar d} \le e^{-\frac{\sigma^2}{2}C \log N},\; \forall i\in [N].
     \end{equation}
We then calculate the probability of the event that all $d_i \le (1-\sigma)$. Since the degrees are independent, this probability is upper bounded by the following probability:
     \begin{align}
           \sum_{i=1}^{N} P(d_i \le (1-\sigma)\bar d)
          = N e^{-\frac{\sigma^2}{2} \kappa \log N} = e^{(1-\frac{\sigma^2}{2}C) \log N}.
     \end{align}
Therefore, the lower bound for all degrees can be derived as follows: 
\begin{equation}
    \begin{aligned}
            \mathbb P(d_i \ge (1-\sigma)\bar d)  \ge 1- e^{(1-\frac{\sigma^2}{2}\kappa) \log N} 
         =1 - N^{(1-\frac{\sigma^2}{2}\kappa)},\;
            \text{for all}\; 1\le i \le N,
        \end{aligned}
\end{equation} 
The desired lower bound is obtained by setting $\kappa = \frac{2(r+1)}{\sigma^2}$. Similarly, the upper bound can be derived, resulting in another $\kappa(\sigma, r)$. Selecting the larger value of $\kappa$ completes the proof.

\end{proof}
Afterward, we provide a concentration inequality for the signed adjacency matrix. Denote the expectation of $\mathcal{A}$ by $\mathbb E_{\mathcal{G}}[\mathcal{A}]$.

We introduce the following Lemma.
\begin{lemma}\label{lemma:2}
    \cite{oliveira2009concentration} Let $\mathbf{X}_1, ...,\mathbf{X}_n \in \mathbb R^{n \times n}$ be symmetric independent random matrices defined on common probability space with zero means and $\mathbf{Z}=\sum_{i=1}^n \mathbf{X}_i$. There exists a constant $m$ such that when the spectral norm of $\mathbf{X}_i$ is bounded by $\mathbf{X}_i\le m$, for all $1\le i \le n$, then for any $\sigma>0$ $\bar d>0$, 
    \begin{equation}
        P(\Vert \mathbf{Z}\Vert \ge \sigma \bar d) \le 2Ne^{-\frac{\sigma^2\bar d^2}{8\lambda^2+4m\sigma \bar d}},
    \end{equation}
    where $\lambda^2=\Vert \sum_{i=1}^n \mathbb E[\mathbf{X}_i^2]\Vert$.
\end{lemma}
The following corollary provides an upper bound for the spectral norm of $\mathcal{A}-\mathbb E_{\mathcal{G}}[\mathcal{A}]$.
 \begin{corollary}\label{corollary:2}
    For any $\sigma>0$ and $r>0$, there exists a constant $\kappa(\sigma, r)$ such that when $\bar d \ge \kappa \log N$, the following holds with probability at least $1- N^{-r}$,
    \begin{equation}\nonumber
        \mathbb P(\Vert \mathcal{A}-\mathbb E_{\mathcal{G}}[\mathcal{A}]\Vert \le \sigma \bar d) \ge 1-2N^{-r}.
        \end{equation}
\end{corollary}

\begin{proof}
    Let $\left\{\mathbf{e}\right\}_{i=1}^N$ be the canonical basis for $\mathbb{R}^{N}$. We can rewrite $\mathcal{A}-\mathbb E_{\mathcal{G}}[\mathcal{A}]=\sum_{1\le i\le j \le N} \mathbf{X}_{i,j}$, where
    \begin{equation}
\mathbf{X}_{i,j}=\left\{
	\begin{aligned}
	 (I_{i,j}-b_{i,j}) L_{i,j} (\mathbf{e}_i \mathbf{e}_j^T +  \mathbf{e}_j \mathbf{e}_i^T)\quad i\neq j\\
	 (I_{i,j}-b_{i,j}) \mathbf{e}_i \mathbf{e}_i^T\quad i=j
	\end{aligned}
	\right
	.
\end{equation}
Let $L_{i,j}=1$ and $b_{i,j}=p$ if $v_i$ and $v_j$ belong to the same class; otherwise, let $L_{i,j}=-1$ and $b_{i,j}=q$. Furthermore, let $I_{i,j} \sim \text{Bern}(b_{i,j})$.

Since $\mathbb E[I_{i,j}-b_{i,j}]=0$ and the edges are sampled independently, Lemma \ref{lemma:2} can be applied here. By observing that
\begin{equation}
    \Vert \mathbf{X}_{i,j}\Vert \le \Vert\mathbf{e}_i \mathbf{e}_j^T +  \mathbf{e}_j \mathbf{e}_i^T \Vert =1, 
\end{equation}
we may set $m=1$.

Notice that $\mathbf{e}_i^T \mathbf{e_j}=0$ if $i\neq j$, the variances can be computed as follows.
\begin{equation}
\mathbb E[\mathbf{X}_{i,j}^2]=\left\{
	\begin{aligned}
	 b_{i,j}(1-b_{i,j})  (\mathbf{e}_i  \mathbf{e}_i^T +  \mathbf{e}_j \mathbf{e}_j^T)\quad i\neq j\\
	 b_{i,j}(1-b_{i,j}) \mathbf{e}_i \mathbf{e}_i^T\quad i=j
	\end{aligned}
	\right
	.
\end{equation}
Then we have
\begin{equation}
    \begin{aligned}
    \Vert\sum_{1\le i \le j} E[\mathbf{X}_{i,j}^2] \Vert&= \Vert\sum_{i} b_{i,i}(1-b_{i,i}) \mathbf{e}\mathbf{e}^T + \sum_{i < j} b_{i,j}(1-b_{i,j})  (\mathbf{e}_i  \mathbf{e}_i^T +  \mathbf{e}_j \mathbf{e}_j^T)\Vert\\ 
    & = \Vert\sum_{i} b_{i,i}(1-b_{i,i}) \mathbf{e}\mathbf{e}^T + \sum_{i \neq j} b_{i,j}(1-b_{i,j})  \mathbf{e}_i  \mathbf{e}_i^T\Vert\\ 
    &= \Vert\sum_{i,j} b_{i,j}(1-b_{i,j})  \mathbf{e}_i  \mathbf{e}_i^T \Vert\ \le \Vert\sum_{i,j} b_{i,j} \mathbf{e}_i  \mathbf{e}_i^T \Vert \\ 
    & = \Vert\sum_i \bar d \mathbf{e}_i  \mathbf{e}_i^T \Vert = \bar d
\end{aligned}
\end{equation}

Now we can apply Lemma \ref{lemma:2} with $m=1$ and $\lambda^2=\bar d$, which gives
\begin{equation}
    \begin{aligned}
     \mathbb P(\Vert \mathcal{A}-\mathbb E_{\mathcal{G}}[\mathcal{A}]\Vert \ge \sigma \bar d) \le 2Ne^{-\frac{\sigma^2\bar d^2}{8 \bar d + 4\sigma \bar d}} 
     \le 2e^{(1-\frac{\sigma^2 \kappa }{8  + 4\sigma })\log N} = 2 N ^ {1-\frac{\sigma^2 \kappa }{8  + 4\sigma }}
\end{aligned}
\end{equation}
Setting $\kappa=\frac{(r+1)(8+4\sigma)}{\sigma^2}$, we get
\begin{equation}\nonumber
        \mathbb P(\Vert \mathcal{A}-\mathbb E_{\mathcal{G}}[\mathcal{A}]\Vert \le \sigma \bar d) \ge 1-2N^{-r},
\end{equation}
which completes the proof.
\end{proof}

Our concentration inequality for $\bar {{\mathbf{h}}}_a^{(k)} - \bar {{\mathbf{h}}}_b^{(k)}$ is given by the following theorem.
 \begin{theorem}\label{theorem: main}
    For any $\sigma>0$, and $r>0$, there exists a constant $\kappa=(\sigma, r)$ such that when $\bar d\ge \kappa \log N$, 
    \begin{equation}\nonumber
        \Vert (\bar {{\mathbf{h}}}_a^{(k)} - \bar {{\mathbf{h}}}_b^{(k)})-  \mathbb{E}_{\mathcal{G}}[\bar {{\mathbf{h}}}_a^{(k)} - \bar {{\mathbf{h}}}_b^{(k)} ]  \Vert \le 2k\sigma \sqrt{\frac{2C}{N}} \Vert \mathbf{U} \Vert
    \end{equation}
    holds with probability at least $1-O(N^{-r})$.
\end{theorem}

\begin{proof}
Let $\mathbb{E}_{\mathcal{G}}[\mathcal{P}] = \mathbb{E}_{\mathcal{G}} [{\mathcal{D}}^{-\frac{1}{2}} {\mathcal{A}} {\mathcal{D}}^{-\frac{1}{2}}]$ be the expectation of $\mathcal{P} = \mathcal{D}^{-\frac{1}{2}} \mathcal{A} \mathcal{D}^{-\frac{1}{2}}$. Our objective is to bound $\Vert(\bar {{\mathbf{h}}}_a^{(k)} - \bar {{\mathbf{h}}}_b^{(k)})-  \mathbb{E}_{\mathcal{G}}[\bar {{\mathbf{h}}}_a^{(k)} - \bar {{\mathbf{h}}}_b^{(k)} ]\Vert$. To achieve this, we need to bound $\Vert\mathcal{P}^k - \mathbb{E}_{\mathcal{G}}[\mathcal{P}^k]\Vert$.

We consider the following spectral norm:
    \begin{equation}
        \Vert \mathcal{D}^{\frac{1}{2}} \mathbb{E}_{\mathcal{G}}[{\mathcal{D}}^{-\frac{1}{2}}]\Vert = \text{max}_{1\le i \le N} \frac{\sqrt{d_i}}{\sqrt{\bar d}}.
    \end{equation}
By Corollary \ref{corollary:1}, we have 
\begin{equation}\label{eq: 27}
    \mathbb P (\Vert \mathcal{D}^{\frac{1}{2}} \mathbb{E}_{\mathcal{G}}[{\mathcal{D}}^{-\frac{1}{2}}]\Vert \le \sqrt{1+\sigma}) \ge 1 - N^{-r}.
\end{equation}
Similar to \citet{oliveira2009concentration}, we introduce an intermediate operator $\mathcal{O}$ defined as
\begin{equation}
    \begin{aligned}
        \mathcal{O}=\mathbb{E}_{\mathcal{G}}[{\mathcal{D}}^{-\frac{1}{2}}]\mathcal{A}\mathbb{E}_{\mathcal{G}}[{\mathcal{D}}^{-\frac{1}{2}}] &= (\mathcal{D}^{\frac{1}{2}} \mathbb{E}_{\mathcal{G}}[{\mathcal{D}}^{-\frac{1}{2}}]) (\mathcal{D}^{-\frac{1}{2}}\mathcal{A}\mathcal{D}^{-\frac{1}{2}}) (\mathcal{D}^{\frac{1}{2}} \mathbb{E}_{\mathcal{G}}[{\mathcal{D}}^{-\frac{1}{2}}])\\
        & = (\mathcal{D}^{\frac{1}{2}} \mathbb{E}_{\mathcal{G}}[{\mathcal{D}}^{-\frac{1}{2}}]) \mathcal{P} (\mathcal{D}^{\frac{1}{2}} \mathbb{E}_{\mathcal{G}}[{\mathcal{D}}^{-\frac{1}{2}}])
    \end{aligned}
\end{equation}
We then compute the difference between $\mathbf{P}$ and $\mathbf{O}$.
\begin{equation}
    \begin{aligned}
        \Vert \mathcal{P} - \mathcal{O}\Vert &= \Vert (\mathcal{D}^{\frac{1}{2}} \mathbb{E}_{\mathcal{G}}[{\mathcal{D}}^{-\frac{1}{2}}])\mathcal{P} (\mathcal{D}^{\frac{1}{2}} \mathbb{E}_{\mathcal{G}}[{\mathcal{D}}^{-\frac{1}{2}}]) -\mathcal{P} \Vert\\
        & = \Vert (\mathcal{D}^{\frac{1}{2}} \mathbb{E}_{\mathcal{G}}[{\mathcal{D}}^{-\frac{1}{2}}] -\mathbf{I}_N) \mathcal{P} \mathcal{D}^{\frac{1}{2}} \mathbb{E}_{\mathcal{G}}[{\mathcal{D}}^{-\frac{1}{2}}] + \mathcal{P} (\mathcal{D}^{\frac{1}{2}} \mathbb{E}_{\mathcal{G}}[{\mathcal{D}}^{-\frac{1}{2}}] -\mathbf{I}_N) \Vert \\
        &\le \Vert (\mathcal{D}^{\frac{1}{2}} \mathbb{E}_{\mathcal{G}}[{\mathcal{D}}^{-\frac{1}{2}}] -\mathbf{I}_N) \mathcal{P} \mathcal{D}^{\frac{1}{2}} \mathbb{E}_{\mathcal{G}}[{\mathcal{D}}^{-\frac{1}{2}}] \Vert + \Vert \mathcal{P} (\mathcal{D}^{\frac{1}{2}} \mathbb{E}_{\mathcal{G}}[{\mathcal{D}}^{-\frac{1}{2}}] -\mathbf{I}_N) \Vert \\
        & \le \Vert (\mathcal{D}^{\frac{1}{2}} \mathbb{E}_{\mathcal{G}}[{\mathcal{D}}^{-\frac{1}{2}}] -\mathbf{I}_N)\Vert \cdot \Vert \mathcal{P} \Vert \cdot \Vert\mathcal{D}^{\frac{1}{2}} \mathbb{E}_{\mathcal{G}}[{\mathcal{D}}^{-\frac{1}{2}}] \Vert + \Vert \mathcal{P} \Vert \cdot \Vert \mathcal{D}^{\frac{1}{2}} \mathbb{E}_{\mathcal{G}}[{\mathcal{D}}^{-\frac{1}{2}}] -\mathbf{I}_N \Vert \\
        &\le (\sqrt{1+\sigma} -1) (\sqrt{1+\sigma}) + (\sqrt{1+\sigma}-1) \\
        & = (\sqrt{1+\sigma} -1) (\sqrt{1+\sigma} + 1) = \sigma
        .
    \end{aligned}  
\end{equation}
The inequality $\Vert \mathcal{P} \Vert \le 1$ follows immediately from the fact that the spectrum of any Laplacian $\mathbf{L} = \mathbf{I} - \mathbf{P}$ lies in the interval $[0,2]$, including signed Laplacian~\citep{atay2014spectrum}. The last inequality is obtained by utilizing Eq. (\ref{eq: 27}). Consequently, it leads to
\begin{equation}\label{eq:right}
    \mathbb P(\Vert \mathcal{P} - \mathcal{O}\Vert \le \sigma) \ge 1-N^{-r}.
\end{equation}

We proceed with the proof by comparing $\mathcal{O}$ to $\mathbb{E}_{\mathcal{G}}[{\mathcal{P}}]$
\begin{equation}
      \mathcal{O} - \mathbb{E}_{\mathcal{G}}[ {\mathcal{P}}] = \mathbb{E}_{\mathcal{G}}[{\mathcal{D}}^{-\frac{1}{2}} ](\mathcal{A}-\mathbb{E}_{\mathcal{G}}[{\mathcal{A}}]) \mathbb{E}_{\mathcal{G}}[{\mathcal{D}}^{-\frac{1}{2}}] = \sum_{i\le j} \frac{\mathbf{X}_{i,j}}{\bar d} = \frac{1}{\bar d} (\mathcal{A}-\mathbb{E}_{\mathcal{G}}[{\mathcal{A}}]).
\end{equation}
Applying Corollary \ref{corollary:2} we have 
\begin{equation}\label{eq:left}
    \mathbb P(\Vert \mathcal{O}- \mathbb{E}_{\mathcal{G}}[{\mathcal{P}}]\Vert \le \sigma)  \ge 1 - 2N^{-r}
\end{equation}
Combining (\ref{eq:right}), (\ref{eq:left}), and 
\begin{equation}
    \Vert \mathcal{P} - \mathbb{E}_{\mathcal{G}}[ {\mathcal{P}} ]\Vert = \Vert \mathcal{P} - \mathcal{O} + \mathcal{O} - \mathbb{E}_{\mathcal{G}}[ {\mathcal{P}}]\Vert \le \Vert \mathcal{P} - \mathcal{O}\Vert + \Vert \mathcal{O}- \mathbb{E}_{\mathcal{G}}[{\mathcal{P}}]\Vert.
\end{equation}
It leads to an upper bound for the norm as follows:
\begin{equation}
    \mathbb P( \Vert \mathcal{P} - \mathbb{E}_{\mathcal{G}}[] {\mathcal{P}}] \Vert \le 2\sigma) \ge 1- O(N^{-r}).
\end{equation}
Note that to ensure the simultaneous validity of both results, here $\kappa$ should be set to the larger value from Corollaries \ref{corollary:1} and \ref{corollary:2}.

It remains to show that the bound grows linearly with the number of layers. To achieve this, we rewrite
\begin{equation}
\begin{aligned}
    &\mathcal{P}^k - \mathbb{E}_{\mathcal{G}}[{\mathcal{P}}^k] = ( \mathcal{P} - \mathbb{E}_{\mathcal{G}}[ {\mathcal{P}}]) (\mathcal{P}^{k-1}+\mathcal{P}^{k-2}\mathbb{E}_{\mathcal{G}}[ {\mathcal{P}}]+...+\mathbb{E}_{\mathcal{G}}[ {\mathcal{P}}^{k-1}]) \\
    & = ( \mathcal{P} - \mathbb{E}_{\mathcal{G}}[ {\mathcal{P}}])\mathcal{P}^{k-1} +( \mathcal{P} - \mathbb{E}_{\mathcal{G}}[ {\mathcal{P}}])\mathcal{P}^{k-2}\mathbb{E}_{\mathcal{G}}[ {\mathcal{P}}]...+ ( \mathcal{P} - \mathbb{E}_{\mathcal{G}}[ {\mathcal{P}}])\mathbb{E}_{\mathcal{G}}[ {\mathcal{P}}^{k-1}].
\end{aligned}
\end{equation}
Since the terms $\mathcal{P}^{k-1} $,$\mathcal{P}^{k-2}\mathbb{E}_{\mathcal{G}}[ {\mathcal{P}}]$,..., $\mathbb{E}_{\mathcal{G}}[ {\mathcal{P}}^{k-1}]$ have norms no more than $1$, we have
\begin{equation}
    \begin{aligned}
        \Vert\mathcal{P}^k - \mathbb{E}_{\mathcal{G}}[ {\mathcal{P}}^k]\Vert \le \Vert  \mathcal{P} - \mathbb{E}_{\mathcal{G}}[ {\mathbf{P}}] \Vert + \ldots + \Vert  \mathcal{P} - \mathbb{E}_{\mathcal{G}}[ {\mathbf{P}}] \Vert = k  \Vert  \mathcal{P} - \mathbb{E}_{\mathcal{G}}[ {\mathbf{P}}] \Vert.
    \end{aligned}
\end{equation}
In turn, we get
\begin{equation}
    \mathbb P( \Vert \mathcal{P}^k - \mathbb{E}_{\mathcal{G}}[ {\mathcal{P}}^k] \Vert \le 2k\sigma) \ge 1- O(N^{-r}).
\end{equation}

WLOG, we assume the nodes are properly ordered such that 
    \begin{equation}
        \mathbf{U} = \left ( 
        \begin{matrix}
            \mathbbm 1_{N/C}\mathbf{u}_1 \\
              \vdots  \\
            \mathbbm 1_{N/C}\mathbf{u}_C
        \end{matrix}
        \right ),
    \end{equation}
where $\mathbbm 1_{N/C}$ is a all $1$ column vector of length $N/C$. For convenience, we further define a vector $\zeta$ as follows.
\begin{equation}
    \zeta = \left ( 
        \begin{matrix}
            \mathbbm 1_{N/C} \\
            -\mathbbm 1_{N/C} \\
             0\\
             \vdots \\
             0
        \end{matrix}
        \right ).
\end{equation}

Recall that $\mathbb{E}_{\mathcal{G}}[ {\mathcal{P}}^k]$ is the k-th power of the expected propagation operator. We have
\begin{equation}
    \mathbb{E}_{\mathcal{G}}[\bar {\mathbf{\mathbf{h}}}_a^{(k)} - \bar {\mathbf{\mathbf{h}}}_b^{(k)} ]  = \frac{C}{N} \zeta^T  \mathbb{E}_{\mathcal{G}}[ {\mathcal{P}}^k] \mathbf{U},
\end{equation}
and similarly
\begin{equation}
     \bar {\mathbf{\mathbf{h}}}_a^{(k)} - \bar {\mathbf{\mathbf{h}}}_b^{(k)}  = \frac{C}{N} \zeta^T   {\mathcal{P}}^k \mathbf{U}
\end{equation}
Hence, we can rewrite
\begin{equation}
    (\bar {\mathbf{\mathbf{h}}}_a^{(k)} - \bar {\mathbf{\mathbf{h}}}_b^{(k)}) - \mathbb{E}_{\mathcal{G}}[\bar {\mathbf{\mathbf{h}}}_a^{(k)} - \bar {\mathbf{\mathbf{h}}}_b^{(k)} ]  = \frac{C}{N} \zeta^T   ({\mathcal{P}}^k - \mathbb{E}_{\mathcal{G}}[ {\mathcal{P}}^k])\mathbf{U}
\end{equation}
Our final result can be derived as follows:
\begin{equation}
    \begin{aligned}
        \Vert  (\bar {\mathbf{\mathbf{h}}}_a^{(k)} - \bar {\mathbf{\mathbf{h}}}_b^{(k)}) - \mathbb{E}_{\mathcal{G}}[\bar {\mathbf{\mathbf{h}}}_a^{(k)} - \bar {\mathbf{\mathbf{h}}}_b^{(k)} ]  \Vert &= \Vert \frac{C}{N} \zeta^T   ({\mathcal{P}}^k - \mathbb{E}_{\mathcal{G}}[ {\mathcal{P}}^k])\mathbf{U} \Vert\\
        & \le \frac{C}{N}\Vert \zeta \Vert \Vert \mathcal{P}^k - \mathbb{E}_{\mathcal{G}}[ {\mathbf{P}}^k] \Vert \Vert \mathbf{U}\Vert  = 2k\sigma \sqrt{\frac{2C}{N}} \Vert \mathbf{U} \Vert,
    \end{aligned}
\end{equation}
which proves the claim.
\end{proof}

\subsection{Proof of Lemma \ref{lem:desirablem2m}}

\begin{proof}
    We will prove Lemma \ref{lem:desirablem2m} by induction. Without loss of generality, we use sum as the pooling function. We prove that, for any node $v_i$ and any $d$, the following holds.
    \begin{equation}
        \mathbf h_i^{(k+d)} = \rho^{(d,k)}(\mathbf{H}^{(k)}, \Phi^{(k+d)}(\cdot);i)=||_{t=1}^{C^d} \Phi^{(k+d)}(\Gamma_{i,t}^{(k+d)}),
    \end{equation}
     where $\Phi^{(k+d)}(\Gamma_{i,t}^{(k+d)}) = \sum \mathbf h_{j}^{(k)}\mathcal{W}^{(k+d)}, \Gamma_{i,t}^{(k+d)}=\{\mathbf h_{j}^{(k)}: dis(v_i, v_j)=d, \;t=\sum_{p=1}^{d}(y_p-1)C^{d-p}+1,
     \;y_{j}=((t+2)\mod 3) +1 \}$. Here $v_1, v_2,... v_d$ is a path from $v_1$ to $v_{j}$ of length $d$, with $v_1\in \mathcal{N}(v_i)$ and $v_d=v_{j}$, $y_p$ is the label of $v_p$, $\forall p \in [d]$ . 

    \textbf{Base case}: $d=1$. It is trivial since $\rho^{(1, k)} = f_{mp}^{(k+1)}$ and $f_{mp}^{(k+1)}$ is desirable by definition.
    
    By definition, $\rho^{(1,k)}(\mathbf{H}_i^{(k)}, \Phi^{(k+1)}(\cdot);i)=\mathbf{h}_i^{(k+1)}= f_{mp}^{(k+1)}(\mathbf{H}^{(k)}, \phi^{(k+1)}(\cdot);i) = ||_{t=1}^C \phi^{(k+1)}(\mathcal \mathcal S_{i,t}^{(k+1)})$, where $\phi^{(k+1)}(\mathcal  S_{i,t}^{(k+1)}) = \sum \mathbf h_j^{(k)} \mathbf{W}^{(k+1)},  \mathcal  S_{i,t}^{(k+1)} = \{\mathbf h_j^{(k)}:v_j\in \mathcal{N}(v_i), y_j=t\}$. Here, $\Gamma_{i,t}^{(k+1)}=\mathcal S_{i,t}^{(k+1)}$, and $\Phi^{(k+1)}(\cdot)$ is equivalent to $\phi^{(k+1)}(\cdot)$, defined as $\phi^{(k+1)}(\mathcal S) = \sum_{s \in S} \mathbf s \mathbf{W}^{(k+1)}$, for any set $\mathcal S$.

    \textbf{Inductive step}: Assume this holds for $d$ and any node $v_j$, i.e., 
    \begin{equation}\label{eq:proof_lemma_34}
        \mathbf h_j^{(k+d)} = \rho^{(d,k)}(\mathbf{H}^{(k)},\Phi^{(k+d)} ;j)=||_{t=1}^{C^d} \Phi^{(k+d)}(\Gamma_{j,t}^{(k+d)}),
    \end{equation}
     where $\Phi^{(k+d)}(\Gamma_{j,t}^{(k+d)}) = \sum \mathbf h_{j'}^{(k)}\mathcal{W}^{(k+d)}, \Gamma_{j,t}^{(k+d)}=\{\mathbf h_{j'}^{(k)}: dis(v_j, v_{j'})=d, \;t=\sum_{p=1}^{d}(y_p-1)C^{d-p}+1,
     \;y_{j'}=((t+2)\mod 3) +1 \}$. Here $v_1, v_2,... v_d$ is a path from $v_1$ to $v_{j'}$ of length $d$, with $v_1\in \mathcal{N}(v_j)$ and $v_d=v_{j'}$, $y_p$ is the label of $v_p$, $\forall p \in [d]$ .

    We shall show that this holds for any node $v_i$ and $d+1$.
    By definition of desirable one-hop \textit{m-2-m} message passing,  
    \begin{equation}
        \mathbf{h}_i^{(k+d+1)}= f_{mp}^{(k+d+1)}(\mathbf{H}_i^{(k+d)}, \phi^{(k+d+1)}(\cdot); i) = ||_{t=1}^C\phi^{(k+d+1)}(\mathcal{S}_{i,t}^{(k+d+1)}),
    \end{equation}
    where $\mathcal{S}_{i,t}^{(k+d+1)} = \{\mathbf h_j^{(k+d)}: v_j \in \mathcal{N}(v_i), y_j =t \}$. We further define $N_{i,t}=\{v_j: v_j \in  \mathcal{N}(v_i), y_j =t\}$. $N_{j,t}^{(d)} = \{v_{j'}: \mathbf h_{j'} \in  \Gamma_{j,t}^{(k+d)}\}$.
    Replacing $\mathbf h_j^{(k+d)}$ with Eq.~\eqref{eq:proof_lemma_34} and defining $\phi^{(k+d+1)}(\cdot)$ as the multiplication of a projection matrix on the right, followed by sum pooling, we get
    \begin{align}
        \mathbf{h}_i^{(k+d+1)} = ||_{t=1}^{C} \sum_{\mathbf{h}_j^{(k+d)}\in \mathcal{S}_{i,t}^{(k+d+1)} } \mathbf{h}_j^{(k+d)} \mathcal{W}^{(k+d+1)} &=  ||_{t=1}^{C} \sum_{v_j\in N_{i,t}} (||_{t=1}^{C^d}\sum_{v_{j'}\in N_{j,t}^{(d)}} \mathbf h_{j'}^{(k)}\mathcal{W}^{(k+d)} )\mathbf{W}^{(k+d+1)}\\
        =||_{t=1}^{C^{d+1}} \sum_{v_{j'}\in N_{j,t}^{(d)} ,v_j\in N_{i,t}} \mathbf h_{j'}^{(k)}\mathcal{W}^{(k+d+1)} &= ||_{t=1}^{C^{d+1}} \Phi^{(k+d+1)}(\Gamma_{i,t}^{(k+d+1)}),
    \end{align}
    where $\Phi^{(k+d+1)}(\Gamma_{i,t}^{(k+d+1)}) = \sum \mathbf h_{j'}^{(k)}\mathcal{W}^{(k+d+1)}, \Gamma_{i,t}^{(k+d+1)}=\{\mathbf h_{j'}^{(k)}: dis(v_i, v_{j'})=d+1,\; t=\sum_{p=1}^{d+1}(y_p-1)C^{d+1-p}+1,\; y_{j'} = ((t+2)\mod 3) +1 \}$. Here $v_1, v_2,... v_{d+1}$ is a path from $v_1$ to $v_{j'}$ of length $d+1$, with $v_1\in \mathcal{N}(v_i)$ and $v_{d+1}=v_{j'}$, $y_p$ is the label of $v_p$, $\forall p \in [d+1]$. Note that when $\mathbf{W}^{(k+d+1)}$ is a block diagonal matrix of the form
    \begin{equation}
        \mathbf{W}^{(k+d+1)} = 
        \left[
\begin{array}{ccc c}
    \mathbf{W} & \mathbf{0} & \cdots  & \mathbf{0}
    \\
     \mathbf{0} &  \mathbf{W} &\cdots &\mathbf{0}  \\
    \vdots &\vdots &\ddots &\vdots \\
    \mathbf{0} &\mathbf{0} &\cdots &  \mathbf{W}
\end{array}
\right],
    \end{equation}
    where $C^d$ $\mathbf{W}$'s with proper dimensions are on the diagonal,
    the following equality holds:
    \begin{equation}
        (||_{t=1}^{C^d}\sum_{v_{j'}\in N_{j,t}^{(d)}} \mathbf h_{j'}^{(k)}\mathcal{W}^{(k+d)} )\mathbf{W}^{(k+d+1)} = ||_{t=1}^{C^d}\sum_{v_{j'}\in N_{j,t}^{(d)}} (\mathbf h_{j'}^{(k)}\mathcal{W}^{(k+d)} \mathbf{W})
    \end{equation}
    Therefore, we have $\mathcal{W}^{(k+d+1)}=\mathcal{W}^{(k+d)}\mathbf{W}$. 

    We have shown that $\mathbf{h}_i^{(k+d+1)}=||_{t=1}^{C^{d+1}} \Phi^{(k+d+1)}(\Gamma_{i,t}^{(k+d+1)})$, $\Gamma_{i,t}^{(k+d+1)}=\{\mathbf h_{j'}^{(k)}: dis(v_i, v_{j'})=d+1,\; t=\sum_{p=1}^{d+1}(y_p-1)C^{d+1-p}+1,\; y_{j'} = ((t+2)\mod 3) +1 \}$ and $\Phi^{(k+d+1)}(\mathcal{S})=\sum_{s\in \mathcal{S}}\mathbf{s}\mathcal{W}^{(k+d)}\mathbf{W}$. All nodes in each set $\Gamma_{i,t}^{(k+d+1)}$ have the same label, and $\Gamma_{i,t}^{(k+d+1)}$ consists of the representations of $v_i$'s $d+1$-hop neighbors, at the $k$-th layer. Combining these concludes the induction.
\end{proof}



\subsection{Proof of Theorem \ref{thm:smp_in_multi_class}}
\begin{proof}
    By matrix multiplication, we can express $\mathcal{T}_{ij}$ as $\mathcal{T}_{ij} = \sum_{p} \prod_{k=1}^{K} \mathcal{A}_{v_{k+1} v_{k}}^{(k)}$. In this expression, $p$ represents a path of length $K+1$ from $v_j$ to $v_i$, where $v_1=v_j$, $v_{K+1}=v_i$, and $v_{k+1}\in \mathcal{N}(v_{k})$,  $\forall k \in [K]$.

    To prove Theorem \ref{thm:smp_in_multi_class}, it suffices to provide a counterexample where all $\mathcal{A}_{v_{k+1} v_{k}}^{(k)}$ are desirable, but $\mathcal{T}_{ij}$ is not desirable. Let's consider a scenario with $K=3$, $v_1v_2v_3$, $y_1=1,y_2=2,y_3=3$, and where there is only one such path. Since all the connected node pairs are heterophilic, the coefficients are all negative. Specifically, $\mathcal{T}_{ij} = \mathcal{A}_{v_3v_2}^{(2)}\mathcal{A}_{v_2v_1}^{(1)}$. As a result, $sign(\mathcal{T}_{ij}) = sign(\mathcal{A}_{v_3v_2}^{(2)})sign(\mathcal{A}_{v_2v_1}^{(1)})=-1\times -1=1$. However, $v_i$ and $v_j$ have diverse labels, thus $\mathcal{T}$ is not desirable.

    The key point here is that the sign of $\mathcal{T}_{ij}$ is determined by the number of heterophilic node pairs along the path from $v_i$ to $v_j$ (an even number of heterophilic node pairs results in positive $\mathcal{T}_{ij}$ while odd number results in negative $\mathcal{T}_{ij}$). This aligns with the binary-class cases: after flipping a two-sided coin an even number of times, it will still land on the same side. However, this is not the case when we have multiple classes: after rolling a dice an even number of times, it could land on any of the possible numbers.
\end{proof}


\subsection{Proof of Lemma \ref{theorem:discri_pow} and Lemma \ref{thm:app_dis_pow}}\label{sec:proof_for_oversmoothing}
We first present the formal form as follows:

      Assume non-empty multisets $\mathcal{X}_{a}$ and $\mathcal{X}_{b}$ each of which consists of $x_{i} \in \mathbb{R}^{d}$ vectors.
     Let $\phi : 2^\mathcal{X} \mapsto \mathbb{R}^{d}$ be one of the elementwise-sum, -mean, or -max function. 
     Consider a function $f : {\mathcal{X}} \mapsto \mathbb{R}^{\omega d}$ such that 
    $f(\mathcal{X}) = \Vert_{k=1}^{\omega}\phi(\mathcal{X}'_{k})$, where $\omega\in \mathbb N_+$, $\mathcal{X}'_{k} \subseteq \mathcal{X}$ are any disjoint subsets whose union is a full set.
    For message vectors: $\mathbf{m}_{a}^{\textit{m2m}} \coloneqq f(\mathcal{X}_{a})$, $\mathbf{m}_{b}^{\textit{m2m}} \coloneqq f(\mathcal{X}_{b})$ and $\mathbf{m}^{\textit{m2e}}_{a},  \coloneqq \phi(\mathcal{X}_{a})$,  $\mathbf{m}^{\textit{m2e}}_{b},  \coloneqq \phi(\mathcal{X}_{b})$, the following holds:
     $\Vert \mathbf{m}_{a}^{\textit{m2m}} - \mathbf{m}_{b}^{\textit{m2m}}  \Vert_2 \ge \Vert \mathbf{m}_{a}^{\textit{m2e}} - \mathbf{m}_{b}^{\textit{m2e}}  \Vert_2$.

    (1) Any two different multisets that can be discriminated by \textit{m-2-e}, i.e., $\mathbf{m}_{a}^{m2e} \neq \mathbf{m}_{b}^{m2e}$, can be discriminated by \textit{m-2-m} as well: $\mathbf{m}_{a}^{m2m} \neq \mathbf{m}_{b}^{m2m}$.

     (2) There exist two different multisets, such that $\mathbf{m}_{a}^{m2m}\neq \mathbf{m}_{b}^{m2m}$, but \textit{m-2-e} is unable to discriminate between them: $\mathbf{m}_{a}^{m2e} = \mathbf{m}_{b}^{m2e}$. 

     (3) We always have $\Vert \mathbf{m}_{a}^{m2m} - \mathbf{m}_{b}^{m2m}  \Vert_2^2 \ge \Vert \mathbf{m}_{a}^{m2e} - \mathbf{m}_{b}^{m2e}  \Vert_2^2$.

\begin{proof}
    We prove (1) by contradiction. Assume there exist two different multisets that \textit{m-2-m} cannot discriminate between them, i.e., $\mathbf{m}_{a}^{m2m} = \mathbf{m}_{b}^{m2m}$, but \textit{m-2-e} determines they are different: $\mathbf{m}_{a}^{m2e} \neq \mathbf{m}_{b}^{m2e}$. $\mathbf{m}_{a}^{m2m} = \mathbf{m}_{b}^{m2m}$ indicates that $\phi(\mathcal{X}'_{a,k})=\phi(\mathcal{X}'_{b,k})$, $\forall k \in [\mathcal{\omega}]$ (assume the indices have been properly reordered). Let $\phi(\cdot)$ be the sum pooling. We have
    \begin{equation}
\phi(\mathcal{X}'_{a,k})=\phi(\mathcal{X}'_{b,k}) \Rightarrow \sum_{\mathbf x\in \mathcal{X}'_{a,k}}\mathbf x = \sum_{\mathbf x\in \mathcal{X}'_{b,k}}\mathbf x
    \end{equation}

We then sum over these subsets
\begin{equation}
    \sum_{k=1}^{\mathcal{\omega}} \phi(\mathcal{X}'_{a,k}) = \sum_{k}^{\mathcal{\omega}} \sum_{\mathbf x\in \mathcal{X}'_{a,k}}\mathbf x = \sum_{\mathbf x\in \mathcal{X}'_{a}} \mathbf x.
\end{equation}
Since $\phi(\mathcal{X}'_{a,k})=\phi(\mathcal{X}'_{b,k})$, $\forall k \in [\mathcal{\omega}]$ , we can deduce
\begin{equation}
     \sum_{k=1}^{\mathcal{\omega}} \phi(\mathcal{X}'_{a,k}) = \sum_{k=1}^{\mathcal{\omega}} \phi(\mathcal{X}'_{b,k}) \Rightarrow  \sum_{\mathbf x\in \mathcal{X}'_{a}} \mathbf x =  \sum_{\mathbf x\in \mathcal{X}'_{b}} \mathbf x.
\end{equation}
That is to say $\phi(\mathcal{X}_{a})=\phi(\mathcal{X}_{b})$, which equals $\mathbf{m}_{a}^{m2e} = \mathbf{m}_{b}^{m2e}$. This contradicts to the assumption $\mathbf{m}_{a}^{m2e} \neq \mathbf{m}_{b}^{m2e}$. Hence, the assumption does not hold.

We can apply the same technique to prove the case of mean pooling with a minor modification. In the case of \textit{m-2-m} mean pooling, the pooling operation is not performed within each subset individually. Specifically, \begin{equation}
\phi(\mathcal{X}'_{a,k}) = \frac{1}{|\mathcal{X}_a|}\sum_{\mathbf x \in \mathcal{X}'_{a,k}} \mathbf x. 
\end{equation} 
The subsequent steps remain the same as in the case of sum pooling.

Similarly, other weighted summation schemes can be handled in a similar manner by utilizing the same coefficient for every $\mathbf x$ as in \textit{m-2-e}. which implies that $\sum_{k=1}^{\mathcal{\omega}} \phi(\mathcal{X}'_{a,k})=\phi(\mathcal{X}_{a})$ and $\sum_{k=1}^{\mathcal{\omega}} \phi(\mathcal{X}'_{b,k})=\phi(\mathcal{X}_{b})$. The key point here is the linearity property exhibited by the weighted average.

We now consider $\phi(\cdot)$ as an element-wise max pooling. As $\phi(\mathcal{X}_{a})\neq \phi(\mathcal{X}_{b})$, there must exist a component $c$ for which we have $\mathbf x \in \mathcal{X}_{a}$ and $\mathbf{x}' \in \mathcal{X}_{b}$ such that $\mathbf x_c\neq \mathbf x'_c$. These values correspond to the maximum values of the $c$-th component among all vectors in their respective multisets. Without loss of generality, assuming $\mathbf x$ is in $\mathcal{X}'_{a,1}$ and $\mathbf x_c>\mathbf x'_c$, there does not exist a subset $\mathcal{X}'_{b,k}$ of $\mathcal{X}_{b}$ such that the $c$-th component of $\phi(\mathcal{X}'_{b,k})$ equals $\mathbf x_c$.  Since we are guaranteed to have at least one different element, $\mathbf{m}_{i}^{m2m} \neq \mathbf{m}_{i'}^{m2m}$ is always ensured.

The proof of (2) can be demonstrated through examples. For sum and mean pooling, let us consider two multisets, namely $\{1,3\}$ and $\{2,2\}$ (one-dimensional vector). In the case of sum and mean pooling, both multisets are deemed equivalent since $1+3=2+2$ and $\frac{1+3}{2}=\frac{2+2}{2}$. However, for \textit{m-2-m} pooling, they can be distinguished by placing two elements in different subsets. For max pooling, a similar example $\{1,3\}$ and $\{2,3\}$ can be used.

The proof of (3) follows directly the triangle inequality. For sum and mean pooling and other weighted summation schemes, we have
\begin{equation}
    \sum_{k=1}^{\mathcal{\omega}} \phi(\mathcal{X}'_{a,k})=\phi(\mathcal{X}_{a}), \quad \sum_{k=1}^{\mathcal{\omega}} \phi(\mathcal{X}'_{b,k})=\phi(\mathcal{X}_{b})
\end{equation}
and 
\begin{equation}
    \mathbf{m}_a^{m2m} = [\phi(\mathcal{X}'_{a,1})||\phi(\mathcal{X}'_{a,2})||\ldots||\phi(\mathcal{X}'_{a,\mathcal{\omega}})], \; \mathbf{m}_{b}^{m2m} = [\phi(\mathcal{X}'_{b,k})||\phi(\mathcal{X}'_{b,k})||\ldots||\phi(\mathcal{X}'_{b,\mathcal{\omega}})]
\end{equation}
By triangle inequality, we get
\begin{equation}
\begin{aligned}
     \Vert \mathbf{m}_a^{m2m} - \mathbf{m}_{b}^{m2m} \Vert_2^2 =  \sum_{k=1}^{\mathcal{\omega}} \Vert\phi(\mathcal{X}'_{a,k}) - \phi(\mathcal{X}'_{b,k})\Vert_2^2 &\ge \Vert \sum_{k=1}^{\mathcal{\omega}} (\phi(\mathcal{X}'_{a,k}) - \phi(\mathcal{X}'_{b,k}))\Vert_2^2 \\
     = 
    \Vert \phi(\mathcal{X}_{a})- \phi(\mathcal{X}_{b}) \Vert &= \Vert \mathbf{m}_a^{m2e} - \mathbf{m}_{b}^{m2e} \Vert_2^2. 
\end{aligned}
\end{equation}

In the case of max pooling, the subsets of \textit{m-2-m} also contain the maximum value of each component. Without loss of generality, we can rearrange the vectors such that the maximum values of all the components of $\mathcal{X}_{a}$ and $\mathcal{X}_{b}$ are located in $\mathcal{X}'_{a,1}$ and $\mathcal{X}'_{b,1}$, respectively. In other words, we have $\phi(\mathcal{X}'_{a,1})=\phi(\mathcal{X}_{a})$ and $\phi(\mathcal{X}'_{b,1})=\phi(\mathcal{X}_{b})$. It follows that
\begin{equation}
    \Vert \mathbf{m}_a^{m2m} - \mathbf{m}_{b}^{m2m} \Vert_2^2 =  \sum_{k=1}^{\mathcal{C}} \Vert\phi(\mathcal{X}'_{a,k}) - \phi(\mathcal{X}'_{b,k})\Vert_2^2 \ge \Vert \phi(\mathcal{X}'_{a,1}) - \phi(\mathcal{X}'_{b,1})\Vert_2^2= 
    \Vert \phi(\mathcal{X}_{a}) - \phi(\mathcal{X}_{b})\Vert = \Vert \mathbf{m}_a^{m2e} - \mathbf{m}_{b}^{m2e} \Vert_2^2.
\end{equation}
This completes the proof.

\end{proof}

\subsection{Proof of Theorem \ref{thm:oversmoothing}}
We first present the formal form of Theorem \ref{thm:oversmoothing}:

\textbf{Theorem \ref{thm:oversmoothing}}  (Formal). Let $\mathcal{G}$ be a graph (a random variable) generated under the CSBM model (Def.~\ref{def:csbm}).
    Denote the set of nodes belonging to each class $c\in[C]$ as 
    $\mathcal{V}_{c} \coloneqq \{v_{j} \in \mathcal{V} : y_{j} = c\}$.
    Denote the mean of $K$-th layer embeddings of $\mathcal{V}_{c}$ 
    as $\bar{\mathbf{h}}^{(K)}_{c} \coloneqq \frac{C}{N}\sum_{v_{i} \in \mathcal{V}_{c}}\mathbf{h}^{(K)}_{i}$.
    Assume any two classes $a, b\in [C]$, such that at the $k-1$-th layer, their expected class means converge to the same point:
    \begin{equation}
        \Vert  E_{\mathcal{G}} [\bar{\mathbf{h}}_a^{(k-1)}]-  E_{\mathcal{G}} [\bar{\mathbf{h}}_b^{(k-1)}]\Vert=0.
    \end{equation}
    Let $ E_{\mathcal{G}} [\bar{\mathbf{h}}_a^{(k)}]$ and $ E_{\mathcal{G}} [\bar{\mathbf{h}}_b^{(k)}]$ be the expected means obtained by using Eq.~\eqref{eq:m2m_message_passing}, with $\phi(\cdot)$ being a mean pooling.
    Then, we have:
    \begin{equation}
        \Vert\mathbb E_{\mathcal{G}} [\bar{\mathbf{h}}_a^{(k)}] -  \mathbb{E}_{\mathcal{G}} [\bar{\mathbf{h}}_b^{(k)}]\Vert >0,
    \end{equation}
if $\Vert \mathbb{E}_{\mathcal{G}} [\bar{\mathbf{h}}_a^{(k-1)}]\Vert \neq 0$ or $ E_{\mathcal{G}} [\bar{\mathbf{h}}_b^{(k-1)}] \neq 0$ and $p\neq q$.

\begin{proof}
According to the definition of Eq.~\eqref{eq:m2m_message_passing}, $\mathbb E_{\mathcal{G}} [\bar{\mathbf{h}}_a^{(k)}]$ can be calculated as:
\begin{equation}
    \mathbb{E}_{\mathcal{G}} [\bar{\mathbf{h}}_a^{(k)}] =\frac{1}{p+(\mathcal{C}-1)q} [\cdots||p \mathbb E_{\mathcal{G}} [\bar{\mathbf{h}}_a^{(k-1)}] || q\mathbb E_{\mathcal{G}} [\bar{\mathbf{h}}_b^{(k-1)}]||\cdots]
\end{equation} 
Similarly, we get
\begin{equation}
     \mathbb E_{\mathcal{G}} [\bar{\mathbf{h}}_b^{(k)}] = \frac{1}{p+(\mathcal{C}-1)q}[\cdots||q \mathbb E_{\mathcal{G}} [\bar{\mathbf{h}}_a^{(k-1)}] || p\mathbb E_{\mathcal{G}} [\bar{\mathbf{h}}_b^{(k-1)}]||\cdots]
\end{equation} 
Then, $\Vert\mathbb E_{\mathcal{G}} [\bar{\mathbf{h}}_a^{(k)}] - \mathbb E_{\mathcal{G}} [\bar{\mathbf{h}}_b^{(k)}]\Vert$ is derived as follows:
\begin{equation}
\begin{aligned}
    \Vert\mathbb E_{\mathcal{G}} [\bar{\mathbf{h}}_a^{(k)}] - \mathbb E_{\mathcal{G}} [\bar{\mathbf{h}}_b^{(k)}]\Vert 
    &= \frac{1}{p+(\mathcal{C}-1)q} \Vert[\cdots||(p-q)\mathbb E_{\mathcal{G}} [\bar{\mathbf{h}}_a^{(k-1)}] || (q-p) \mathbb E_{\mathcal{G}} [\bar{\mathbf{h}}_b^{(k-1)}]||\cdots ]\Vert\\
    & \ge \frac{|p-q|}{p+(\mathcal{C}-1)q} \Vert  E_{\mathcal{G}} [\bar{\mathbf{h}}_a^{(k-1)}] \Vert + \frac{|p-q|}{p+(\mathcal{C}-1)q} \Vert  E_{\mathcal{G}} [\bar{\mathbf{h}}_b^{(k-1)}] \Vert >0
\end{aligned}
\end{equation}
If we use a \texttt{SMP} as the same as Eq.~\eqref{eq:theoryembedding}, the expected class means can be calculated as follows:
\begin{equation}\label{eq:exp_mean_class_a}
     \mathbb{E}_{\mathcal{G}}[\bar {\mathbf{\mathbf{h}}}_a^{(k)}] = \frac{p \mathbb{E}_{\mathcal{G}}[ \bar {\mathbf{\mathbf{h}}}_a^{(k-1)}] - q\sum_{c\in[C], c \neq a}\mathbb{E}_{\mathcal{G}}[ \bar {\mathbf{\mathbf{h}}}_c^{(k-1)}]}{p+(C-1)q},
\end{equation}
and similarly, 
\begin{equation}\label{eq:exp_mean_class_b}
    \mathbb{E}_{\mathcal{G}}[\bar {\mathbf{\mathbf{h}}}_b^{(k)}] = \frac{p \mathbb{E}_{\mathcal{G}}[ \bar {\mathbf{\mathbf{h}}}_b^{(k-1)}] - q\sum_{c\in[C], c \neq a}\mathbb{E}_{\mathcal{G}}[ \bar {\mathbf{\mathbf{h}}}_c^{(k-1)}]}{p+(C-1)q}.
\end{equation}
By subtracting $\mathbb{E}_{\mathcal{G}}[\bar {\mathbf{\mathbf{h}}}_b^{(k)}]$ from $\mathbb{E}_{\mathcal{G}}[\bar {\mathbf{\mathbf{h}}}_a^{(k)}]$, we get
\begin{equation}
    \mathbb{E}_{\mathcal{G}}[\bar {\mathbf{\mathbf{h}}}_a^{(k)} - \bar {\mathbf{\mathbf{h}}}_b^{(k)} ]= \frac{p+q}{(p+(C-1)q) } \mathbb{E}_{\mathcal{G}}[\bar {\mathbf{\mathbf{h}}}_a^{(k-1)} - \bar {\mathbf{\mathbf{h}}}_b^{(k-1)}]=0.
\end{equation}
Thus, we have shown that \textit{m-2-m} can escape from oversmoothing, while $\texttt{SMP}$ get stuck in oversmoothing if $\mathbb{E}_{\mathcal{G}}[\bar {\mathbf{\mathbf{h}}}_a^{(k-1)}] =\mathbb{E}_{\mathcal{G}}[ \bar {\mathbf{\mathbf{h}}}_b^{(k-1)}]$ provided.

\end{proof}
    
\begin{equation}
\end{equation}

\section{Dataset and Baseline Description}\label{dataset_and_baseline}
\subsection{Datasets}
The citation networks \textbf{Cora, Citeseer,} and \textbf{Pubmed}~\citep{DBLP:conf/icml/YangCS16} are comprised of nodes representing scientific papers and edges representing citation relationships. The node features are represented as bag-of-words, and each label corresponds to the field to which the paper belongs.

\textbf{Actor}~\citep{pei2020geom} is a co-occurrence network derived from the film-director-actor-writer network, where the nodes are characterized by bag-of-words representations of the actors' Wikipedia pages. The network's edges indicate the co-occurrence of two actors on the same web page. The task is to classify actors into five categories

\textbf{Cornell, Texas, and Wisconsin}~\citep{pei2020geom} consist of web pages collected from computer science departments at various universities. In these datasets, the nodes represent web pages, while the edges symbolize the hyperlinks connecting them. We employ bag-of-words representations as feature vectors for the nodes. The objective is to classify the web pages into five categories: student, project, course, staff, and faculty.

\textbf{Chameleon and Squirrel}~\citep{rozemberczki2021multi} are two subgraphs of web pages in Wikipedia. Nodes in this context symbolize web pages, while the edges represent the hyperlinks connecting them. The node features encapsulate various informative nouns extracted from the corresponding Wikipedia pages. The objective of this task is to forecast the average daily traffic received by each web page.

\textbf{Penn94}~\citep{lim2021large} refers to a subgraph obtained from Facebook, focusing on students as its nodes. These nodes possess various features such as their majors, second majors/minors, dormitories/houses, academic years, and high schools. The gender of the students is utilized as the labels for the nodes.

\textbf{Amazon-ratings}~\citep{platonov2022critical} This dataset is based on the Amazon product co-purchasing network metadata dataset2 from SNAP Datasets. Nodes are products (books, music CDs, DVDs, VHS video tapes), and edges connect products that are frequently bought together. The task is to predict the average rating given to a product by reviewers.

\textbf{}

\subsection{Baselines}
\paragraph{GPR-GNN~\citep{chien2020adaptive}:} The mathematical description of the process employed in \textbf{GPR-GNN} can be formulated as follows:
\begin{equation}
\mathbf{Z} = \sum_{k=0}^{K} \gamma_k \mathbf{P}^k \mathbf{H}^{(0)}.
\end{equation}
Here, $\mathbf{P}=\mathbf{D}^{-\frac{1}{2}}\mathbf{A}\mathbf{D}^{-\frac{1}{2}}$ represents the propagation matrix utilized in GCN~\citep{kipf2016semi}, and $\gamma_k$ stands for a learned real value. The propagation coefficient for the connection between nodes $i$ and $j$ is calculated as $\gamma_k \mathbf{P}^k_{ij}$. Since $y_k$ are real values, $\gamma_k \mathbf{P}^k_{ij}$ can be negative.

\paragraph{FAGCN~\citep{bo2021beyond}:} FAGCN combines both low-frequency and high-frequency signals by updating node representation as:
\begin{equation}
    \mathbf{h}_i^{(l)} = \epsilon \mathbf{h}_i^{(l)} + \sum_{i \in \mathcal{N}(i)} \frac{\alpha_{ij}}{d_id_j} \mathbf{h}_j^{(l-1)},
\end{equation}
where $\alpha_{ij}$ represents a learned propagation coefficient lies in $(-1, 1)$.


\paragraph{GGCN:~\citep{yan2022two}} GGCN incorporates signed messages by computing cosine similarity between connected node pairs:
\begin{equation}
    \mathbf{S}^{(l)}_{ij} = Cosine(\mathbf{h}_i^{(l-1)}, \mathbf{h}_j^{(l-1)}).
\end{equation}
GGCN learns a weighted combination of self-representations, positive messages, and negative messages:
\begin{equation}
    \mathbf{h}_i^{(l)} = \text{ReLU} (\beta_1^{(l)}  \mathbf{h}_i^{(l)}+ \beta_2^{(l)}\sum_{j \in \mathcal{N}(i), \mathbf{S}^{(l)}_{ij}>0} \mathbf{S}^{(l)}_{ij} \mathbf{h}_j^{(l)} + \beta_3^{(l)}\sum_{j \in \mathcal{N}(i), \mathbf{S}^{(l)}_{ij}<0} \mathbf{S}^{(l)}_{ij} \mathbf{h}_j^{(l)} )
\end{equation}
\paragraph{ACM-GCN~\citep{luan2022revisiting}:} ACM-GCN includes a pair of filters that retain the low-frequency and high-frequency.
\begin{equation}
    \mathbf{H}^{(l)} = \text{ReLU} (diag(\alpha_H^{(l)}) (\mathbf{I}-\mathbf{P}) \mathbf{H}^{(l-1)} \mathbf{W}^{(l)}_{HP} + diag(\alpha_L^{(l)})\mathbf{P} \mathbf{H}^{(l-1)} \mathbf{W}^{(l)}_{LP} +  diag(\alpha_I^{(l)}) \mathbf{H}^{(l-1)} \mathbf{W}^{(l)}_{I}).
\end{equation}
Here, the signed message is incorporated through the term $\mathbf{I}-\mathbf{P}$.
\paragraph{Goal~\citep{DBLP:conf/icml/ZhengZLZWP23}:} The node representations in the convolution layer of Goal are updated as:
\begin{equation}
    \mathbf{H}^{(l+1)} = \text{ReLU} (( \alpha \mathbf{I} +\beta \mathbf{A}_o -\gamma \mathbf{A}_t - \delta \mathbf{A}_{to})\mathbf{H}^{(l)} \mathbf{W}^{(l)} ),
\end{equation}
where $\mathbf{A}_o$ and $\mathbf{A}_t$ represent the constructed homophily graph and heterophily graph, respectively. $\mathbf{A}_{to}$ represents the combined graph. The utilization of signed messages occurs through the terms $-\gamma \mathbf{A}_t$ and $- \delta \mathbf{A}_{to}$.

\paragraph{AERO-GNN~\citep{DBLP:conf/icml/LeeBYS23}:} AERO-GNN is an extension of GPR-GNN, distinguished by two significant differences. Firstly, while GPR-GNN employs the propagation matrix of GCN, AERO-GNN learns attention coefficients. Secondly, in AERO-GNN, the coefficient $\gamma_k$ operates on a node-wise basis, whereas GPR-GNN adopts a layer-wise $\gamma_k$. Consequently, negative $\gamma_k$ values are employed to introduce signed messages.

\paragraph{H\textsubscript{2}GCN~\citep{zhu2020beyond} :}
There are three designs regarding concatenation in H\textsubscript{2}GCN.
\begin{equation}
    \mathbf{h}_i^{(k)} = [\mathbf{h}_i^{(k-1)}||  \mathbf{m}_i^{(k)}],
\end{equation}
where the ego- and neighbor-information are separated by concatenation.
Second, the direct and second-order neighbor information are separated as it indicates that second-order neighbors expect to contain more homophilic neighbors in heterophilic graphs.
\begin{equation}
    \mathbf{m}_i^{(k)} = [\texttt{Aggr}(\left\{h_j:v_j \in \mathcal{N}(v_i) \right\})||\texttt{Aggr}(\left\{h_j:v_j \in \mathcal{N}_2(v_i) \right\})]
\end{equation}
Lastly, the final out is obtained by combining the representations of different hops.
\begin{equation}
    \mathbf{h}_i^{\texttt{(final)}} = [\mathbf{h}_i^{(1)}||\mathbf{h}_i^{(2)}||...||\mathbf{h}_i^{(K)}]
\end{equation}

\paragraph{Ordered GNN~\citep{song2022ordered} :}
Ordered GNN employs a mechanism where the representations of different hops are separated and addresses the issue of dimension explosion. Specifically, the final representation of a node is obtained as follows:
\begin{equation}
\mathbf{h}_i^{\texttt{(final)}} = \mathbf{h}_i^{(1)}\odot g_i^{(1)} + \mathbf{h}_i^{(2)} \odot (g_i^{(2)} - g_i^{(1)}) \ldots+\mathbf{h}_i^{(K)}\odot (g_i^{(K)}-g_i^{(K-1)}),
\end{equation}
where $g_i^{(k)}$ is a vector with a split point $p_i^{(k)}$, such that the first $p_i^{(k)}$ elements of $g_i^{(k)}$ are ones while the remaining $1-p_i^{(k)}$ elements are zeros. These $p_i^{(k)}$ values satisfy $p_i^{(k+1)} \geq p_i^{(k)}$. In some cases, when the models become deeper, $g_i^{(k)}$ at shallower layers approach being all-one vectors. Consequently, the deep representations $\mathbf{h}_i^{(k+1)}$ in later layers are rarely encoded in the final representation $\mathbf{h}_i^{\texttt{(final)}}$, as $g_i^{(k+1)} - g_i^{(k)}\approx \mathbf{0}$ when $g_i^{(k)}$ is nearly an all-one vector. This can be confirmed by referring to the visualization results presented in the original paper.

\paragraph{DMP~\citep{yang2021diverse}:} DMP leverages a channel-wise attention mechanism:
\begin{equation}
    \mathbf h_i^{(l)} = \text{ReLU}(\mathbf c_{ii}^{(l)}\odot\mathbf h_i^{(l-1)} + \sum_{j\in \mathcal{N}(i)}\mathbf c_{ij}^{(l)}\odot\mathbf h_j^{(l-1)}),
\end{equation}
where $\odot$ represents the element-wise product of vectors, and $\mathbf{c}_{ij}^{(l)} = \text{tanh}([\mathbf{h}_i^{(l-1)} || \mathbf{h}_j^{(l-1)}] \mathbf{W}_c^{(l)})$ is a vector with the same length as $\mathbf{h}_j^{(l-1)}$. The activation function $\text{tanh}(c\cdot)$ suggests that the message can become negative.

\section{Experimental Setup}
We implemented M2M-GNN using PyTorch~\citep{paszke2019pytorch}. We applied LayerNorm~\citep{ba2016layer} after each layer. Adam optimizer~\citep{kingma2014adam} was used for optimizing the model parameters. Our experiments were conducted on an NVIDIA Tesla V100 with 32GB of memory. For dataset splitting, we utilized 10 random splits with training, validation, and testing proportions of $48\%$, $32\%$, and $20\%$, respectively, as provided in \citet{pei2020geom}. These splits are available from PyTorch Geometry \citep{fey2019fast}. If the results of the baseline methods on these benchmark datasets are publicly available, we report them directly. Otherwise, we executed these methods using the codes released by their authors and fine-tuned the hyperparameters based on the validation set. To identify the optimal hyperparameters, a grid search was performed on the validation set, and the search space is summarized in Table \ref{tab:hyper}.

\begin{table*}[]
\footnotesize
\setlength\tabcolsep{1.65pt}
    \centering
    \begin{tabular}{c|c}
   \toprule
       \textbf{Hyperparameter} &\textbf{Range} \\
    \midrule
    learning rate &$\left\{ 0.005, 0.01, 0.02 \right\}$\\
    weight decay &$\left\{ 5e-4,5e-3,1e-2 \right\}$\\
    dropout &$\left[0,0.7\right]$\\
    early stopping &$\left\{40, 100, 200\right\}$\\
    $\beta$ &$\left [0.0,0.9\right]$\\
    $\lambda$ &$\left [0.0,1.0\right ]$\\
     hidden channels $d$ \ &$\left \{32, 64, 128, 256 \right \}$\\
     $\alpha$ &$(0,1)$  \\
      $\tau$ &$\left\{0.1,0.5, 1.0, 2.0\right\}$\\
     number of layers  &$\left\{2,4,8 \right\}$\\
    \bottomrule
    \end{tabular}
    \caption{Search space for hyperparameters.}
    \label{tab:hyper}
\end{table*}

\section{Time Complexity}\label{app:time_complexity}
This section presents an analysis of the time complexity of M2M-GNN. M2M-GNN involves several operations. Firstly, in each layer, a linear transformation $\mathbf{h}_i^{(k)}\mathbf{W}^{(k)}$ is performed using matrix multiplication between two matrices of sizes $N \times d$ and $d \times d/C$. This operation has a time complexity of $\mathcal{O}(Nd^2/C)$. Next, the computation of $\mathbf s^{(k)}(i,j)$ requires a vector addition with a time complexity of $\mathcal{O}(|\mathcal{E}|d/C)$, as well as a matrix multiplication with a time complexity of $\mathcal{O}(|\mathcal{E}| d)$. Following that, a message vector is constructed for each node according to the equation:

\begin{equation}
\mathbf{m}^{(k)}_{i} = \lVert_{t=1}^{\mathcal{C}}\mathbf{C}_{it}^{(k)}, \quad \mathbf{C}_{it}^{(k)} = \sum_{v_{j} \in N_{i}}\mathbf{s}_{t}^{(k)}(i,j)\hat{\mathbf{h}}^{(k)}_{j} \in \mathbb{R}^{d/\mathcal{C}}, \forall t \in [\mathcal{C}]
\end{equation}

The time complexity of the computation of $\mathbf{C}_{it}^{(k)}$ is $\mathcal{O}(|\mathcal{E}|d/\mathcal{C})$. Since we have $\mathcal{C}$ chunks, the total time complexity for constructing message vectors is $\mathcal{O}(|\mathcal{E}|d)$. Additionally, the combination of $\mathbf m_i^{(k)}$ and $\mathbf{h}_i^{(0)}$ costs $\mathcal{O}(Nd)$. By combining all these operations, the derived time complexity of M2M-GNN is $\mathcal{O}(Nd^2/\mathcal{C} +|\mathcal{E}| d)$.

In comparison, the time complexity of the vanilla GCN~\citep{kipf2016semi} is $\mathcal{O}(Nd^2 + |\mathcal{E}| d)$. It is worth noting that the time complexity of our proposed M2M-GNN is equivalent to that of GCN.

\section{Impact of Non-linearity}\label{app:non-linearity}
The ReLU activation function has been known to exacerbate over-smoothing. In this study, we investigate the scenario of a GNN utilizing an activation function after $k$ linear graph convolutions. We analyze the difference $\Vert \text{ReLU}(\bar{ \mathbf{h}}_a^{(k)})-\text{ReLU}(\bar{\mathbf{h}}_b^{(k)})\Vert$, and show that $\Vert \bar{\mathbf{h}}_{a}^{(k)}- \bar{\mathbf{h}}_{b}^{(k)}\Vert\ge \Vert \text{ReLU}(\bar{\mathbf{h}}_{a}^{(k)}) - \text{ReLU}(\bar{\mathbf{h}}_{b}^{(k)}) \Vert$. Given that this activation operates element-wise, our analysis centers on the $i$-th element:
\begin{equation}
    \text{ReLU}(\bar{\mathbf{h}}_{a,i}^{(k)})-\text{ReLU}(\bar{\mathbf{h}}_{b,i}^{(k)}) = \begin{cases}
        0,\quad &\bar{\mathbf{h}}_{a,i}^{(k)}<0, \bar{\mathbf{h}}_{b,i}^{(k)}<0
        \\
        -\bar{\mathbf{h}}_{b,i}^{(k)}, \quad &\bar{\mathbf{h}}_{a,i}^{(k)}<0, \bar{\mathbf{h}}_{b,i}^{(k)}>0\\
        \bar{\mathbf{h}}_{a,i}^{(k)}, \quad &\bar{\mathbf{h}}_{a,i}^{(k)}, \bar{\mathbf{h}}_{b,i}^{(k)}<0 \\
        \bar{\mathbf{h}}_{a,i}^{(k)} - \bar{\mathbf{h}}_{b,i}^{(k)}, \quad &\bar{\mathbf{h}}_{a,i}^{(k)}>0, \bar{\mathbf{h}}_{b,i}^{(k)}>0 \\
    \end{cases}.
\end{equation}

\begin{equation}
    |\bar{\mathbf{h}}_{a,i}^{(k)}-  \bar{\mathbf{h}}_{b,i}^{(k)}|- |\text{ReLU}(\bar{\mathbf{h}}_{a,i}^{(k)})-\text{ReLU}(\bar{\mathbf{h}}_{b,i}^{(k)}| = \begin{cases}
        |\bar{\mathbf{h}}_{a,i}^{(k)} -  \bar{\mathbf{h}}_{b,i}^{(k)}|,\quad &\bar{\mathbf{h}}_{a,i}^{(k)}<0, \bar{\mathbf{h}}_{b,i}^{(k)}<0
        \\
        |\bar{\mathbf{h}}_{a,i}^{(k)}|, \quad &\bar{\mathbf{h}}_{a,i}^{(k)}<0, \bar{\mathbf{h}}_{b,i}^{(k)}>0\\
        |\bar{\mathbf{h}}_{b,i}^{(k)}|, \quad &\bar{\mathbf{h}}_{a,i}^{(k)}>0, \bar{\mathbf{h}}_{b,i}^{(k)}<0 \\
        0, \quad &\bar{\mathbf{h}}_{a,i}^{(k)}>0, \bar{\mathbf{h}}_{b,i}^{(k)}>0 \\
    \end{cases}.
\end{equation}
Therefore, we always have $|\bar{\mathbf{h}}_{a,i}^{(k)} -  \bar{\mathbf{h}}_{b,i}^{(k)}|\ge |\text{ReLU}(\bar{\mathbf{h}}_{a,i}^{(k)})-\text{ReLU}(\bar{\mathbf{h}}_{b,i}^{(k)})|$, which leads to $(\bar{\mathbf{h}}_{a,i}^{(k)}-  \bar{\mathbf{h}}_{b,i}^{(k)})^2\ge (\text{ReLU}(\bar{\mathbf{h}}_{a,i}^{(k)})-\text{ReLU}(\bar{\mathbf{h}}_{b,i}^{(k)}))^2$. Consequently, the following relationship can be derived: $\Vert \bar{\mathbf{h}}_{a}^{(k)}- \bar{\mathbf{h}}_{b}^{(k)}\Vert = \sqrt{\sum_{i} (\bar{\mathbf{h}}_{a,i}^{(k)}-  \bar{\mathbf{h}}_{b,i}^{(k)})^2} \ge \sqrt{\sum_{i} (\text{ReLU}(\bar{\mathbf{h}}_{a,i}^{(k)})-\text{ReLU}(\bar{\mathbf{h}}_{b,i}^{(k)}))^2} = \Vert \text{ReLU}(\bar{\mathbf{h}}_{a}^{(k)}) - \text{ReLU}(\bar{\mathbf{h}}_{b}^{(k)}) \Vert$. This suggests that the ReLU activation function consistently increases the similarity between representations, thereby exacerbating over-smoothing.

Regarding Theorem~\ref{thm:smp_in_multi_class}, our conclusion remains valid even when nonlinear activation functions are applied. For instance, let's consider an \texttt{SMP} that incorporates a non-linear activation function in updating the representation matrix as $\mathbf H^{(l+1)}=\sigma(\mathcal{A} \mathbf{H}^{(l)} \mathbf{W})$, where $\sigma(\cdot)$ represents various non-linear activations, and $\mathcal{A}$ is a propagation matrix. Under this framework, for a given vector $v$, elements of the vector $\frac{\sigma(v)}{v}$ adhere to the condition $0 \leq \frac{\sigma(v_i)}{v_i} \leq 1$ (equal to 0 if $v_i=0$) for each component $i$. It is worth noting that most common non-linear activation functions, such as ReLU, LeakyReLU, Tanh, and GeLU, conform to this requirement. Consequently, the $i$-th column of $\mathbf H^{(l+1)}$ can be redefined as $\mathbf H^{(l+1)}_{:,i}= \mathbf N_i \mathcal{A} (\mathbf{H}^{(l)} \mathbf{W})_{:,i}$, where $\mathbf N_i$ is a diagonal matrix with elements ranging between 0 and 1, acting as the non-linear activation factor. The modified coefficient matrix $\mathcal{A'}=\mathbf N \mathcal{A}$ inherits the sign of its elements from the original matrix $\mathcal{A}$ due to the non-negativity of the diagonal elements in $\mathbf N_i$. Therefore, non-linear activation does not alter the signs of the elements. According to Definition~\ref{def:good}, the desirability of a matrix depends solely on the signs of its elements. Theorem~\ref{thm:smp_in_multi_class} demonstrates the existence of desirable $\mathcal{A^{(k)}}$ matrices whose product, denoted as $\mathcal{T}$, does not retain the desirability criteria. Furthermore, when any $\mathbf N_i \mathcal{T}$ resulting from applying a non-linear activation function satisfies $sign(\mathbf N_i \mathcal{T})=sign(\mathcal{T})$, it indicates that $\mathbf N_i \mathcal{T}$ is considered an undesirable matrix as defined in Definition~\ref{def:good}.


\end{document}